\documentclass[letterpaper, 10 pt, conference]{source/ieeeconf}

\IEEEoverridecommandlockouts                              % This command is only needed if 
\usepackage{array}
\usepackage[caption=false,font=normalsize,labelfont=sf,textfont=sf]{subfig}
\usepackage{textcomp}
\usepackage{stfloats}
\usepackage{url}
\usepackage{verbatim}
\usepackage{graphicx}
\usepackage{cite}
\usepackage{multirow}
\usepackage{graphicx} % for \graphicspath
\usepackage{xspace}

\usepackage{url}
\usepackage{xcolor}
\usepackage{algorithm}
\usepackage[noend]{algpseudocode} % algorithmic, \State{}, etc

\usepackage{amssymb,amsmath,amsfonts,comment,cite}
\newtheorem{theorem}{Theorem}

\newtheorem{definition}{Definition}

\newtheorem{remark}{Remark}
\newtheorem{example}{Example}
\usepackage{svg}
\usepackage{ifthen}
\newboolean{shortver}
\setboolean{shortver}{true}% for short version 
\title{\LARGE \bf
	Multi-Agent Combinatorial Path Finding \\with Heterogeneous Task Duration
}

\author{Yuanhang Zhang, Xuemian Wu, Hesheng Wang, Zhongqiang Ren*
\thanks{All authors are with Shanghai Jiao Tong University, China. (email: yuanhang0610@gmail.com; zhongqiang.ren@sjtu.edu.cn).} %Use only for final RAL version
}

\newcommand{\blue}{\color{blue}}
\newcommand\abbrCBSS{CBSS-TPG\xspace}
\newcommand\abbrCBXS{CBSS-D\xspace}
\newcommand\abbrMAPF{MAPF\xspace}
\newcommand\abbrMCPF{MCPF\xspace}

\newcommand\abbranony{anonymous\xspace}
\newcommand\abbrMATSPF{MCPF-D\xspace}
\newcommand\abbrTPGD{TPG-D\xspace}

\usepackage{etoolbox}
\makeatletter
\patchcmd{\@makecaption}
  {\scshape}
  {}
  {}
  {}
\patchcmd{\@makecaption}
{\\}
{.\ }
{}
{}
\makeatother

\begin{document}

% 	\markboth{IEEE Robotics and Automation Letters. Preprint Version. Accepted June, 2024}
% {Ren \MakeLowercase{\textit{et al.}}: Multi-agent Target Sequencing Path Finding with Heterogeneous Task Duration} 

	\maketitle

% 		\thispagestyle{empty}
% 		\pagestyle{empty}

% 		\thispagestyle{plain}
% 		\pagestyle{plain}
% 		\pagenumbering{arabic}

	%%%%%%%%%%%%%%%%%%%%%%%%%%%%%%%%%%%%%%%%%%%%%%%%%%%%%%%%%%%%%%%%%%%%%%%%%%%%%%%%
	\begin{abstract}
		% multi-agent target sequencing path finding with task duration(MA-TS-PF-d)
Multi-Agent Combinatorial Path Finding (MCPF) seeks collision-free paths for multiple agents from their initial locations to destinations, visiting a set of intermediate target locations in the middle of the paths, while minimizing the sum of arrival times.
While a few approaches have been developed to handle MCPF, most of them simply direct the agent to visit the targets without considering the task duration, i.e., the amount of time needed for an agent to execute the task (such as picking an item) at a target location.
MCPF is NP-hard to solve to optimality, and the inclusion of task duration further complicates the problem.
% since the task duration needs to be considered when allocating and an agent needs to occupy the target location when executing the task there and thus block other agents' paths.
% before arriving at destinations. This article addresses a generalized version of MCPF called multiagent target sequencing path find with heterogeneous task duration (\abbrMATSPF), where there are some tasks in certain target locations needed to be executed for heterogeneous task duration. 
% The additional task duration brings challenges of not only computing the conflict-free paths but also guaranteeing the optimality when planning the visiting order of targets (i.e target sequencing).
This paper investigates heterogeneous task duration, where the duration can be different with respect to both the agents and targets.
We develop two methods, where the first method post-processes the paths planned by any MCPF planner to include the task duration and has no solution optimality guarantee; and the second method considers task duration during planning and is able to ensure solution optimality.
% leverage conflict-based Steiner search (CBSS) for \abbrMCPF and propose two novel methods called conflict-based Steiner search with temporal plan graph (\abbrCBSS) and  conflict-based Steiner search with task duration (\abbrCBXS). CBSS-TPG uses a temporal plan graph to post-process paths, ensuring conflict-free paths but without guaranteed optimality. On the other hand, CBSS-D adapts heterogeneous task duration in target sequencing, aiming to compute optimal target sequences. Additionally, a new branching rule is proposed for CBSS-D to enhance conflict resolution efficiency. Our tests verify the feasibility of both \abbrCBSS and \abbrCBXS, demonstrating the optimality of CBSS-D over CBSS-TPG and the improved efficiency of the new branching rule.
The numerical and simulation results show that our methods can handle up to 20 agents and 50 targets in the presence of task duration, and can execute the paths subject to robot motion disturbance.
% Finally, we run Gazebo simulation to validate the planned paths are executable and leverage TPG again to deal with the kinematic differences among robots in simulation.

% pure text abstract
% Multi-Agent Combinatorial Path Finding (MCPF) seeks collision-free paths for multiple agents from their initial locations to destinations, visiting a set of intermediate target locations in the middle of the paths, while minimizing the sum of arrival times. While a few approaches have been developed to handle MCPF, most of them simply direct the agent to visit the targets without considering the task duration, i.e., the amount of time needed for an agent to execute the task (such as picking an item) at a target location. MCPF is NP-hard to solve to optimality, and the inclusion of task duration further complicates the problem. This paper investigates heterogeneous task duration, where the duration can be different with respect to both the agents and targets. We develop two methods, where the first method post-processes the paths planned by any MCPF planner to include the task duration and has no solution optimality guarantee; and the second method considers task duration during planning and is able to ensure solution optimality. The numerical and simulation results show that our methods can handle up to 20 agents and 50 targets in the presence of task duration, and can execute the paths subject to robot motion disturbance.

	\end{abstract}
	
% \def\abstractname{Note to Practitioners}
% \begin{abstract}
% 	\input{source/noteToPrac}
% \end{abstract}

% \begin{IEEEkeywords}
% 	Multi-Agent Path Finding, Path Planning, Integrated Planning and Scheduling
% \end{IEEEkeywords}
	
	%%%%%%%%%%%%%%%%%%%%%%%%%%%%%%%%%%%%%%%%%%%%%%%%%%%%%%%%%%%%%%%%%%%%%%%%%%%%%%%%
	
	\graphicspath{{./figures/}}
	
	\section{Introduction}\label{matspf:sec:intro}
	% \IEEEPARstart{M}{ULTIAGENT} ...... 
% A simple introduction to MCPF and \abbrMATSPF. Provide an application or scenario of \abbrMATSPF to show its practicality. Briefly present our work of two methods——\abbrCBSS and \abbrCBXS to solve \abbrMATSPF and their corresponding advantages and disadvantages.

Multi-Agent Path Finding (\abbrMAPF) seeks a set of collision-free paths for multiple agents from their start to goal locations while minimizing the total arrival times.
This paper considers a generalized problem of \abbrMAPF, called Multi-Agent Combinatorial Path Finding (\abbrMCPF), which further requires the agents to visits a set of intermediate target locations before reaching their goals.
% \IEEEPARstart{M}{ULTIAGENT} combinatorial path finding (\abbrMCPF), proposed in \cite{ren2023cbss}, aims to compute a set of collision-free paths for multiple agents, each originating from their respective starts to destinations. In addition, these agents are obliged to visit a series of intermediary targets or goals before reaching their ultimate destinations while adhering to specific agent-target assignment constraints. This article focuses on an extension of \abbrMCPF referred to as "Multiagent Task Sequencing Path Finding with Heterogeneous Task Duration" (\abbrMATSPF). In \abbrMATSPF, certain agents have the additional requirement of spending specific time intervals at particular intermediate targets, which is also called \textit{heterogeneous}~\cite{atzmon2020generalizing} task duration.
\abbrMAPF and \abbrMCPF arises in applications such as manufacturing and logistics.
Consider a fleet of mobile robots in a factory that are tasked to unload finished parts from different machines.
The robots share a cluttered workspace and need to find collision-free paths to visit all the machines as the intermediate targets before going to their destinations to store the parts.
\abbrMCPF naturally arises in such settings to optimize the manufacturing plan.
% \abbrMATSPF exhibits greater practical applicability and real-world relevance when compared to the original \abbrMCPF framework. For instance in the field of logistics, diverse types of robots equipped with varying locomotion capabilities and sensor configurations are assigned to execute specific tasks at predefined target locations, like quadrotors equipped with motion sensors for surveillance and vehicles equipped with lidar sensors for map building. These robots are tasked with planning collision-free paths from their initial positions to their designated destinations while concurrently fulfilling their assigned tasks. Moreover, due to their distinct capabilities and functions, these robots are assigned to different tasks at various target locations, thus introducing the concept of agent-target assignment. Additionally, the time required to complete a task may vary among different robots. All the constraints and requirements above must be taken into consideration when planning paths in \abbrMATSPF.
\abbrMCPF is challenging due to both the collision avoidance between agents as in \abbrMAPF, and target sequencing, i.e., solving Traveling Salesman Problems (TSPs) to find the allocation and visiting orders of targets for all agents.
Both the TSP and the \abbrMAPF are NP-hard to solve to optimality~\cite{yu2013structure_nphard}, and so is \abbrMCPF.

\begin{figure}[tb]
    \centering
    \subfloat%[Toy example of \abbrMATSPF]
    {
        % \label{toy}
        \begin{minipage}[b]{1\linewidth}
            \centering
            \includegraphics[scale=0.27]{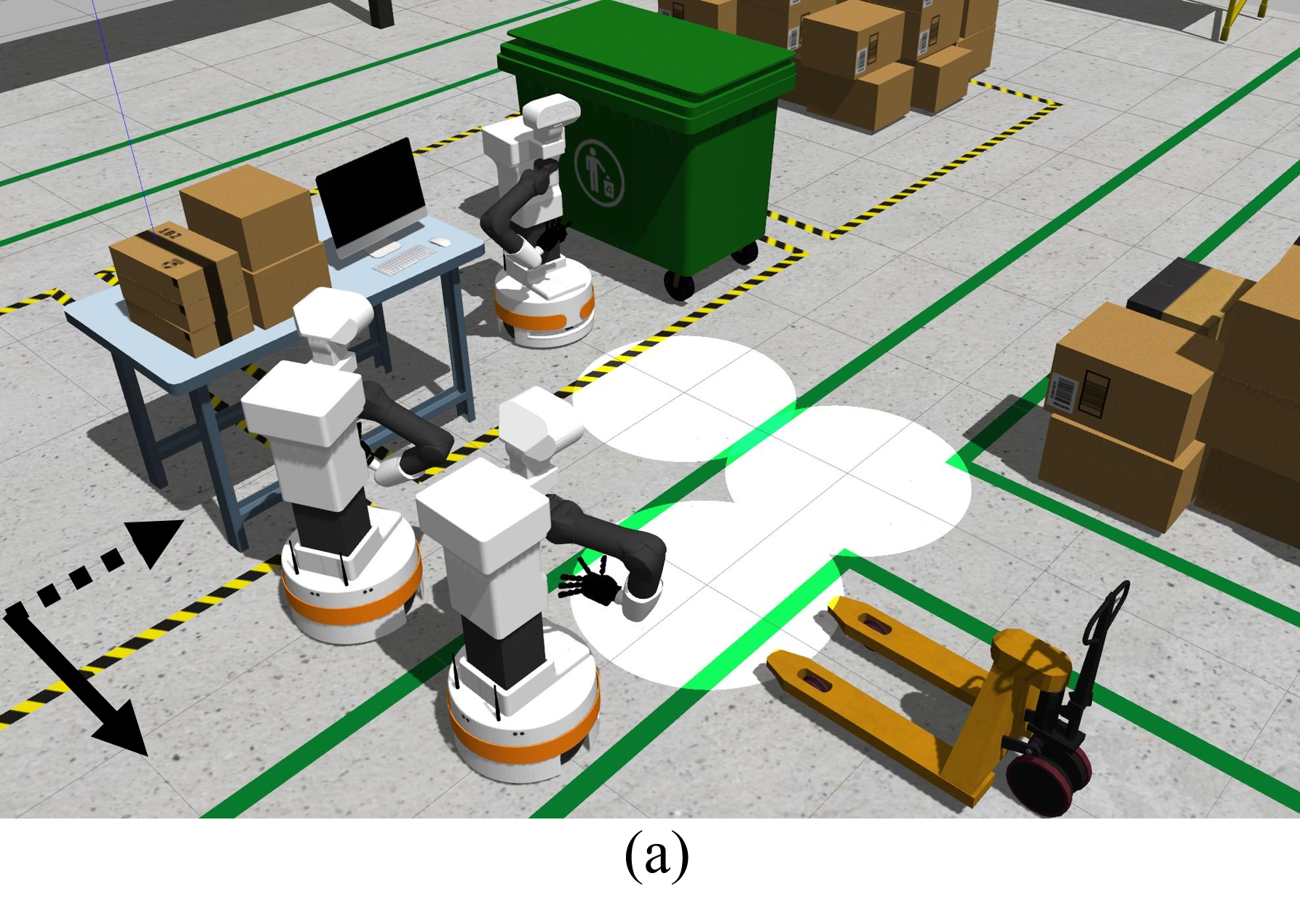}
        \end{minipage}
    }
    \hfill
    \vspace{-0.6cm}
    \subfloat%[The 4$\times$4 grid map overview]
    {
        % \label{toy1}
        \begin{minipage}[b]{1\linewidth}
            \centering
            \includegraphics[scale=0.325]{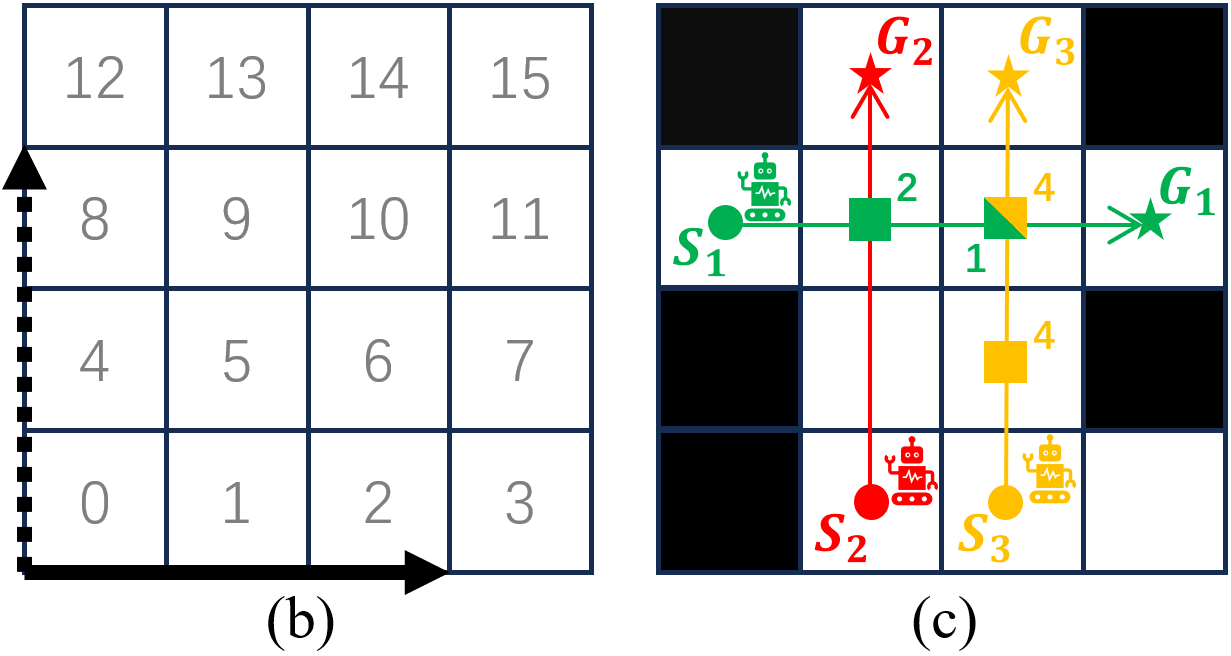}
        \end{minipage}
    }
    % \subfloat%[\abbrMATSPF of 3 agents]
    % {
    %     % \label{toy2}
    %     \begin{minipage}[b]{.41\linewidth}
    %         \centering
    %         \includegraphics[scale=0.825]{source/figures/toy2.png}
    %     \end{minipage}
    % }
    \vspace{-0.3cm}
    \caption{A toy example of \abbrMATSPF.
    The target locations with heterogeneous task duration are marked as the white disks in (a). (c) shows the a 4$\times$4 grid representation of the workspace in (a) and each cell is encoded with a number as shown in (b). There are three targets $6$, $9$, $10$, which are marked as square in (c).
    The color of the targets in (c) shows the assignment constraints. For instance, the task at target 10 can only be executed by the green agent with task duration 1, or by the yellow agent with task duration 4. The $S$ and $G$ shows the initial and goal locations of the agents.}
    \label{fig:toy_example}
    \vspace{-0.5cm}
\end{figure}

Although a few approaches have been developed~\cite{honig2018conflict,ma2016optimal,ren2023cbss,ren2021ms,zhang2022multi,surynek2021multi,9197527,ren2024dms} to handle \abbrMCPF and its variants recently, most of them ignores or simplifies the \emph{task duration} at a target location, which is ubiquitous in practice and is the main focus of this paper.
In other words, when a robot reaches a target to execute the task there, it takes time for the robot to finish the task, and during this period, the robot has to occupy that target location and thus blocks the paths of other agents.
Additionally, when sequencing the targets, the task duration must be considered when solving the TSPs to optimally allocating the targets and finding the visiting order.
Furthermore, task duration can be heterogeneous with respect to the agents and targets: different agents may take different duration for the task at the same target, and for the same agent, different targets may require different duration.

To handle task duration, this paper first formulates a new problem variant of \abbrMCPF called \abbrMATSPF, where D stands for heterogeneous task duration (Fig.~\ref{fig:toy_example}).
\abbrMATSPF generalizes \abbrMCPF and is therefore NP-hard to solve to optimality.
We then develop two methods to solve \abbrMATSPF.
% In the special case of \abbrMATSPF, where all the task duration are set to zero and the problem degrades to a \abbrMCPF problem, our prior work proposed a method named CBSS to compute an optimal solution. In this article, we present two new approaches called conflict-based Steiner search with temporal plan graph (\abbrCBSS) and conflict-based Steiner search with task duration (\abbrCBXS) to solve the general cases of \abbrMATSPF where the task duration are not all zeros. Similar to CBSS, both \abbrCBSS and \abbrCBXS \textit{interleaves} mTSP and MAPF algorithms by alternating between 1) generating new target sequences through solving the mTSP problem with transformation~\cite{5663700} and K-best partition~\cite{VANDERPOORT1999409} and 2) generating conflict-free paths based on the target sequences through resolving conflicts in a manner similar to conflict-based search (CBS)~\cite{sharon2015conflict}. 

The first method has no solution optimality guarantee.
It begins by ignoring the task duration and using an existing planner for \abbrMCPF (such as~\cite{ren2023cbss}) to find a set of paths, and then post-processes the paths to incorporate the task duration while avoiding collision among the agents.
The post-processing leverages \cite{honig2018conflict} to build a temporal planning graph (TPG) that captures the precedence requirement between the waypoints along the agents' paths, and then add the task duration to the target locations while maintaining the precedence requirement to avoid agent-agent collision.
While being able to take advantage of any existing planner for \abbrMCPF, this first method only finds a sub-optimal solution to \abbrMATSPF and the solution cost can be far away from the true optimum especially in the presence of large task duration.
We instantiate this method by using CBSS as the \abbrMCPF planner and name the resulting algorithm \abbrCBSS.

% However, considering the challenge of optimality and conflict-free guarantee brought by additional task duration, we develop different methods in \abbrCBSS and \abbrCBXS to address these problems. 1) In \abbrCBSS, we utilize temporal plan graph (TPG) to develop a post-process to adjust the initial solution to be conflict-free after considering task duration. Note that the initial solution here is generated by CBSS without considering task duration and the solution of \abbrCBSS is conflict-free but without optimality guaranteed. 2) In \abbrCBXS, we consider task duration in the transformation so that the generated target sequences are optimal which guarantees the further optimality. Besides, to deal with the special conflict in a target with task duration, we develop a new branching rule to generate multiple constraints. The intuition behind is to prevent similar constraints from being generated repeatedly. That's because in the target with task duration where a conflict happens, some agent need to wait for multiple time units, which leads to the constraint of the same agent and same location generating for multiple times. Thus, our new branching rule can enhance the efficiency of resolving conflicts.
To find an optimal solution for \abbrMATSPF, we develop our second method called Conflict-Based Steiner Search with Task Duration (\abbrCBXS), which is similar to CBSS~\cite{ren2023cbss} by interleaving target sequencing and path planning.
\abbrCBXS considers task duration during planning, and \abbrCBXS differs from CBSS as follows.
First, \abbrCBXS solves TSPs with task duration to find optimal target sequences for the agents to visit, and thus modifies the target sequencing part in CBSS.
Second, when an agent-agent collision is detected during path planning, \abbrCBXS introduces a new branching rule, which is based on the task duration, to resolve the collision more efficiently than using the basic branching rule in CBSS.

We test and compare our two approaches using an online dataset~\cite{stern2019multi}.
As shown by our results, both \abbrCBSS and \abbrCBXS can handle up to 20 agents and 50 targets, and the solution cost returned by \abbrCBXS is up to 20\% cheaper than \abbrCBSS especially when the task duration is large.
% with three forms of \abbrMATSPF. We first compare \abbrCBXS with \abbrCBSS solving the first two forms of \abbrMATSPF in two maps. We can tell from the numerical results that the success rate differs negligibly between \abbrCBXS and \abbrCBSS but the cost of the optimal solution generated by \abbrCBXS is always lower than that of \abbrCBSS.
Furthermore, the new branching rule in \abbrCBXS is able to help avoid up to 80\% of the planning iterations needed for collision resolution in comparison with the regular branching rule in CBSS.
% of the  compared to the normal branching rule. 
Finally, we combine \abbrCBXS and TPG to both find high quality paths and execute the paths on robots with motion disturbance or inaccurate task duration, which is validated by our high-fidelity simulation in Gazebo.

% The rest of the article discusses the related work in Sec. \ref{matspf:sec:related}, formulates the problem in Sec. \ref{matspf:sec:problem}, and provides the technical background in
% Sec. \ref{matspf:sec:preli}.
% We then present \abbrCBSS in Sec. \ref{matspf:sec:cbssstn} and \abbrCBXS in Sec. \ref{matspf:sec:cbxs}.
% Numerical results and Gazebo simulation are then presented in Sec. \ref{matspf:sec:result}, with conclusion and future work in Sec. \ref{matspf:sec:conclude}.

	% \subsection{Other Related Work}\label{matspf:sec:related}
	% \input{source/related}

	\section{Problem Formulation}\label{matspf:sec:problem}
	
Let index set $I=\{1,2,\dots,N\}$ denote a set of $N$ agents. All agents move in a shared workspace represented as an finite undirected graph $G=(V,E)$, where the vertex set $V$ represents the possible locations for agents and the edge set $E\subseteq V\times V$ denotes the set of all possible actions that can move an agent between two vertices in $V$.
An edge between $u, v \in V$ is denoted as $(u,v)\in E$, and the cost of an edge $e\in E$ is a positive real number $cost(e)\in(0,\infty)$.
To simplify the problem, we consider the case where each edge has a unit cost, which is equal to the traversal time of that edge.

We use a superscript $i\in I$ over a variable to represent the agent to which the variable relates (e.g., $v^i \in V$ means a vertex related to agent $i$). The start (i.e., initial vertex) of agent $i$ is expressed as $v_o^i \in V$, and $V_o \in V$ denotes the set of the starts of all agents. There are $N$ destination vertices in $G$ denoted by the set $V_d \subseteq V$.
The destination vertices are also called goal vertices.
In addition, let $V_t \subseteq V \backslash \{V_o \cup V_d\}$ denote the set of $M$ (intermediate) target vertices.
% where at least one agent is required to visit and spend certain time executing tasks with no disturbance(other agents are prohibited to enter the corresponding vicinity).
Each target or destination vertex $v \in V_t \cup V_d$ is associated with a task that must be executed by an agent.
For each $v \in V_t \cup V_d$, let $f_A(v)\subseteq I$ denote the subset of agents that are capable to visit $v$ and execute the task at $v$.
% and $\tau(v, i)$ ($v \in V_t$) denote the time duration that agent $i$ needs to spend on executing tasks in the target $v's$ location.
Specifically, for each $v \in V_t \cup V_d$, let \emph{task duration} $\tau^i(v)$, a non-negative integer, denote the amount of time that agent $i \in f_A(v)$ takes to execute the task at $v$.
For the same vertex $v \in V_t \cup V_d$, different agents may take different amount of time to execute the task, and hence heterogeneous task duration.
Any agent $i\in I$ (including $i\notin f_A(v)$) can occupy $v$ along its path without executing the task at $v \in V_t \cup V_d$.

All agents share a global clock.
An agent $i$ has three possible actions $a^i$: move through an edge, wait in the current vertex, or execute the task if the agent is at $v\in V_t \cup V_d$.
Here, both move and wait take a unit time and executing the task takes the amount of time indicated by the task duration.
The cost of an action $cost(a^i)$ is the amount of time taken by that action.
Let $\pi^i(v_1^i,v_k^i) := (a_1^i, a_2^i, \dots, a_\ell^i)$ denote a path for agent $i$ between vertices $v_1^i$ and $v_k^i$ in $G$, where $a_k^i, k=1,2,\dots,\ell$ denote an action of the agent.
Let $g(\pi^i(v_1^i,v_k^i))$ denote the cost of the path, which is the sum of the costs of all the actions taken by the agent in the path: $g(\pi^i(v_1^i,v_k^i))=\sum_{k=1,2,\dots,\ell-1}cost(a_k^i)$.
% \footnote{Here, a path $\pi^i$ is represented by a list of actions of agent $i$, and such a $\pi^i$ can always to converted to a list of vertices $\pi^i=(v^i_1,v^i_2,\dots,v^i_k)$ (not vice versa), where the subscripts indicate the consecutive time steps and the vertices indicate the locations of the agent at the corresponding time steps. For the rest of this paper, we use both representation interchangeably.}

Any two agents $i,j \in I$ are in conflict for any of the following two cases. 
The first case is an \emph{edge conflict} $(i,j,e,t)$, where two agents $i,j\in I$ go through the same edge $e$ from opposite directions between times $t$ and $t+1$.
The second case is a \emph{vertex conflict} $(i, j, v, t)$, where two agents $i,j\in I$ occupy the same vertex $v$ at the same time $t$.
Note that due to task duration, vertex conflicts include the case where an agent $i\in I$ is executing a task at some $v \in V_t \cup V_d$ within the time range $[t,t+\tau^i(v)]$, and another agent $j \in I, j\neq i$ occupies $v$ at some time $t' \in [t,t+\tau^i(v)]$.
% Here, agent $j$ may move to $v$ in the period of $[t,t+\tau^i(v)]$.

% Therefore, the problem set of \abbrMATSPF is denoted as $\mathcal{P}=(G, \tau)$ where $G$ represents finite undirected graph of the map and $\tau$ represents the task duration of each target with its eligible agents. 
The \abbrMATSPF problem aims to find a set of conflict-free paths for the agents such that: (i) the task at any $v\in V_t \cup V_d$ is executed by an eligible agent $i\in f_A(v)$; (ii) each agent $i\in I$ starts its path from $v^i_o$ and ends at a unique goal $u\in V_d$ such that $i\in f_A(u)$; (iii) the sum of the path cost of all agents reaches the minimum.
% An example of \abbrMATSPF problem is shown in Fig.~\ref{fig:toy_example}.

% $Remark \ 1:$ The $cost(v^i_k, v^i_{k+1})$ of a path for agent $i$ between two adjacent vertexs $v_k$ and $v_{k+1}$ depends on not only the time agent $i$ spend on moving from $v_k$ to $v_{k+1}$, but also the task time of $v_{k+1}$ which would add additional cost $\tau^i(v_{k+1})$ to the path cost.

\begin{remark}
    When $\tau^i(v)=0$ for all $v\in V_t\cup V_d, i\in f_A(v)$, \abbrMATSPF becomes the existing \abbrMCPF problem \cite{ren2023cbss}.
    Here, a path $\pi^i$ is represented by a list of actions of agent $i$, and such a $\pi^i$ can always to converted to a list of vertices $\pi^i=(v^i_1,v^i_2,\dots,v^i_k)$ (not vice versa), where the subscripts indicate the consecutive time steps and the vertices indicate the locations of the agent at the corresponding time steps. For the rest of this paper, we use both representation interchangeably.
\end{remark}

	\section{Preliminaries}\label{matspf:sec:preli}
	
\subsection{Conflict-Based Steiner Search}
\label{subsec:CBSS}
% According to \cite{ren2023cbss}, review the pipLine~of Conflict-Based Steiner Search(CBSS) especially for the transformation to find target sequences and high-level search to resolve conflicts.

\subsubsection{Overview}

\cite{ren2023cbss} for \abbrMCPF is shown in Alg.~\ref{alg:CBSS},\footnote{In Alg.~\ref{alg:CBSS}, Line~\ref{alg1:line2} and \ref{alg1:line12} are marked in blue indicating the differences between CBSS and \abbrCBXS, which will be explained in Sec. \ref{matspf:sec:cbxs}.} which interleaves target sequencing and path planning as follows (Fig.~\ref{fig:cbss}).
CBSS creates a complete undirected target graph $G_T=(V_T,E_T,C_T)$ with the vertex set $V_T = Vo \bigcup Vt \bigcup Vd (|V_T| = 2N + M )$ and edge set $E_T$ (Line~\ref{alg1:line1}).
Here, $C_T$ represents a symmetric cost matrix of size $(2N + M ) \times (2N + M )$ that defines the cost of each edge  in $E_T$, which is the minimal path cost between the two vertices in the workspace graph $G$.
CBSS then ignores any conflict between the agents and solve a mTSP on $G_T$ to find target sequences that specify the allocation and visiting order of the targets for each agent (Line~\ref{alg1:line2}, Sec.~\ref{cbssd:sec:method:K-best-seq}).
Next, CBSS fixes the target sequence, plans the corresponding paths, and then resolves conflicts along the paths by using Conflict-Based Search (Line~\ref{alg1:line3}-\ref{alg1:line17}, Sec.~\ref{cbssd:sec:method:conflict_resolution}).
Finally, CBSS alternates between resolving conflicts along paths and generating new target sequences until an optimal solution is found.

\begin{figure}[tb]
    \centering
    \includegraphics[width=0.9\linewidth]{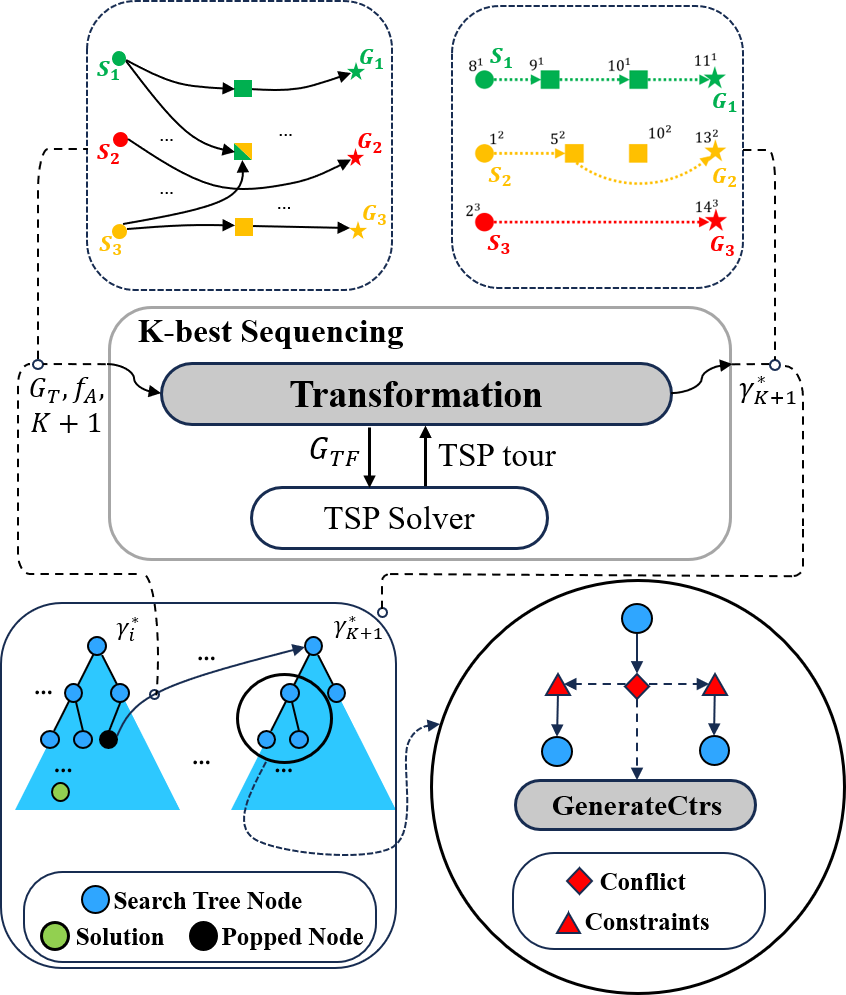}
    \caption{An illustration of CBSS and \abbrCBXS. Both of them generate K-best target sequences and leverage CBS to resolve conflicts between agents. The differences between them are represented by the transformation and constraints generation, which are highlighted by gray-filled text boxes.}
    \label{fig:cbss}
    \vspace{-0.2cm}
\end{figure}

\begin{algorithm}[tb]
    \small
	\caption{Pseudocode for CBSS (CBSS-D)}
    \label{alg:CBSS}
	\begin{algorithmic}[1]
		%
        % \State{\textbf{Input}: The finite undirected graph G = (V, E) representing the shared workspace, and assignment constraint $f_A$}
        % \State{\textbf{Output}: A joint path $\pi_*$ for all agents with minimum cost}
		\State{$G_T=(V_T,E_T,C_T) \gets$ ComputeGraph($G$)} \label{alg1:line1}
		\State{\blue $\gamma^*_1 \gets$ \textit{K-best-Sequencing}($G_T$,$f_A$,$K=1$)} \label{alg1:line2}
		\State{$\Omega \gets \emptyset$} \label{alg1:line3}
		\State{$\pi_{},g_{} \gets$ \textit{LowLevelPlan}($\gamma^*_1$, $\Omega_{}$)}
		\State{Add $P_{root,1}=(\pi,g,\Omega)$ to OPEN}
		\While{OPEN is not empty} %\Comment{Main search loop}
		\State{$P_l=(\pi_l,g_l,\Omega_l) \gets$ OPEN.pop() }
		\State{$P_k=(\pi_k,g_k,\Omega_k) \gets$ \textit{CheckNewRoot}($P_l$, OPEN) }
		\State{$cft \gets$ \textit{DetectConflict}($\pi_k$)}
		\State{\textbf{if} $cft = NULL$ \textbf{then}}
		\State{\indent \textbf{return} $\pi_k$ }
		\State{\blue $\Omega \gets $ \textit{GenerateConstraints}($cft$)} \label{alg1:line12}
		\ForAll{$\omega^i \in \Omega$}
		\State{$\Omega'_k = \Omega_k \cup \{\omega^i\}$}
		%		\State{$\pi_l \gets \pi^k$}
		\State{$\pi'_{k},g'_k \gets$ \textit{LowLevelPlan}($\gamma(P_k)$, $\Omega'_k$)}
		\State{\small// In this LowLevelPlan, only agent $i$'s path is planned.}
		\State{Add $P'_{k}=(\pi'_k,g'_k,\Omega'_k)$ to OPEN} \label{alg1:line17}
		\EndFor
		\EndWhile \label{}
		\State{\textbf{return} failure}
	\end{algorithmic}
\end{algorithm}
\vspace{-0.cm}
\subsubsection{K-best joint sequences}\label{cbssd:sec:method:K-best-seq}
To generate new target sequences, CBSS solves a K-best TSP, which requires finding a set of K cheapest target sequences.
Specifically, let $\gamma^i = \{v_0^i, u_1^i, u_2^i,...,u_l^i,v_d^i\}$ denote a \textit{target sequence} visited by agent $i \in I$, where $v_0^i$ is the start of agent $i$, $u_j^i$ is the $j$-th target visited by agent $i$ with $j=1,...,l$ and $v_d^i \in V_d$ is the goal of agent $i$.
Let $\gamma = \{\gamma^i:i\in I\}$ denote a \textit{joint (target) sequence}, which is a collection of target sequences of all agents.
The cost of a joint sequence is defined as $cost(\gamma):=\sum_{i\in I}cost(\gamma^i)$, where the cost between any two targets $u,v \in \gamma^i$ is set to be the minimum path cost from $u$ to $v$ in $G$.
Here, CBSS seeks a set of K-best joint sequences $\gamma_1,\gamma_2,...,\gamma_k$, whose costs are monotonically non-decreasing: $cost(\gamma_1^*) \leq cost(\gamma_2^*) \leq ... \leq cost(\gamma_k^*)$.

To find K-best joint sequences, CBSS first leverages a transformation method in \cite{junger1995traveling} to convert a mTSP to an equivalent (single agent) TSP so that the existing TSP solver (such as LKH~\cite{lkh}) can be used. The resulting solution to the TSP can then be un-transformed to obtain the joint sequence to the original mTSP.
Additionally, to find K-best joint sequences, CBSS leverages a partition method~\cite{VANDERPOORT1999409,Hamacher1985KBS}, which iteratively creates new TSPs in a systematic way based on the first $(K-1)$-best joint sequences, obtains an optimal solution to each of the new TSPs, and picks the cheapest solution as the $K$-th best solution.

\subsubsection{Conflict Resolution}\label{cbssd:sec:method:conflict_resolution}
For each joint sequence, CBSS uses CBS, a two-level search that creates a search tree, to resolve conflicts and plan paths.
Each node $P$ in a tree is defined as a tuple of $(\pi, g, \Omega)$, where: $\pi = (\pi^1, \pi^2,..., \pi^N)$ is a joint path, a collection of all agents' paths; $g$ is the scalar cost value of $\pi$, i.e., $g=g(\pi)=\sum_{i\in I}g(\pi^i)$; and $\Omega$ is a set of (path) constraints,\footnote{For the rest of this article, we refer to path constraints simply as constraints, which differs from the aforementioned assignment constraint.} each of which is $(i,v,t)$ (or $(i,e,t)$) and indicates that agent $i$ is forbidden to occupy vertex $v$ (or traverse edge $e$) at time $t$.

We first describe CBSS when there is only one joint sequence and then describe when to generate new joint sequences.
With the first computed joint sequence $\gamma_1^*$, a joint path $\pi_1$ is planned by running some low-level (single-agent) planner such as $A^*$ for each agent, while visiting the targets in the same order as in $\gamma_1^*$.
Then, a node corresponding to $\gamma_1^*$ is created, which becomes the root node of the first tree.
If no conflict is detected between agents in $\pi_1$, the search terminates and return $\pi_1$, which is an optimal solution to \abbrMCPF.
Otherwise (i.e., $\pi_1$ has a conflict, say$(i,j,v,t)$), then two new constraints $(i,v,t)$ and $(j,v,t)$ are created as in CBS\cite{sharon2015conflict}.
For each new constraint (say $(i,v,t)$), the low-level planner is invoked for agent $i$ to find a minimum-cost path that satisfies all constraints that have been imposed on agent $i$, and follows the same target sequence as in $\gamma_1^*$.
Then, a corresponding node is generated and added to OPEN, where OPEN is a queue that prioritizes nodes based on their $g$-values from the minimum to the maximum.
In the next search iteration, a node with the minimum cost is popped from OPEN for conflict detection and resolution again.
The entire search terminates until a node with conflict-free joint path is popped from OPEN, which is returned as the solution.

CBSS generates new joint sequences when needed during this search process.
When a node $P=(\pi, g(\pi), \Omega)$ is popped from OPEN.
If $g(\pi)$ is no larger than the cost of a next-best joint sequence (say $cost(\gamma_2^*)$), then, the search continues to detect and resolve conflict as aforementioned.
Otherwise (i.e., $g(\pi) > \gamma_2^*$), a new (root) node (of a new tree) is created, where paths of agents are planned based on $\gamma_2^*$.
Then, the same conflict detection and resolution process follows.
Note that since $cost(\gamma_1^*)$ is a lower bound to the optimal solution cost of \abbrMCPF, and that the nodes are systematically generated and expanded in a best-first search manner, CBSS finds an optimal solution to the \abbrMCPF \cite{ren2023cbss}.

% {\red I reached here. Maybe we need the Pseudo-code of CBSS to explain it well. And in this Pseudo-code, we can also incorporate CBSS-D by showing the lines that are different from CBSS using blue color. Then we only need one Pseudo-code to explain all algorithms, which saves space.}

\subsection{Temporal Plan Graph}
\label{TPG_pre}
To handle kinematic constraints of robots, Temporal Plan Graph (TPG) \cite{honig2016multi} was developed to post-process the joint path output by a \abbrMAPF (or \abbrMCPF) planner.
TPG converts a joint path to a directed acyclic graph $\mathcal{G}=(\mathcal{V},\mathcal{E})$, where each vertex $s\in \mathcal{V}$ represents an event\footnote{$s$ denotes a vertex in $\mathcal{G}$ and $v$ denotes a vertex in $G$.} that an agent enters a location, and each directed edge $e = (s,s^{'}) \in \mathcal{E}$ indicates a temporal precedence between events: event $s{'}$ happen after event $s$. There are two types of edges in $\mathcal{G}$ and the intuition can be described as follows.
Given a conflict-free joint path $\pi = (\pi^1, \pi^2, \dots, \pi^N)$ output by a \abbrMAPF planner.
\begin{itemize}
    \item Type 1: Each agent $i\in I$ enters locations in the same order given by its path $\pi^i$;
    \item Type 2: Any two agents $i,j \in I, i\neq j$ enter the same location in the same order as in their paths $\pi^i$ and $\pi^j$.
\end{itemize}
Let $\pi^i = (v_0^i, v_1^i, ..., v_T^i)$ denote agent $i$'s path in $\pi$, where $s_t^i$ denotes the location agent $i$ reaches at time $t$.
Let $T$ denote the largest arrival time among all agents, and for those agents that arrive at their goals before $T$, their paths are prolonged by letting them wait at their goals until $T$.
A TPG is then constructs as follows.
To create vertices and Type 1 edges in TPG, for each agent $i$, we extract a route $r^i$ from path $\pi^i$ by removing the wait actions (i.e., keeping only the first of the consecutive identical locations in $\pi^i$).
All locations in the routes $r^1, r^2, ..., r^N$ constitute the vertices of TPG, denoted by $s_t^i$ with $i \in I$ and $t \in [0,T]$ indicating that agent $i$ occupies location $v_t^i$ at time $t$. For each pair of two consecutive vertices $s_t^i$ and $s_{t'}^i$ in the route $r^i$, a Type 1 edge $(s_t^i, s_{t'}^i)$ is created, indicating that agent $i$ enters $s_{t}^i$ before entering $s_{t'}^i$. For each pair of two identical locations $s_t^i=s_{t'}^j$ on two different routes ($i\neq j$) with $t < t'$, a Type 2 edge $(s_t^i, s_{t'}^j)$ is created, indicating that agent $i$ enters $s_{t}^i$ before agent $j$ enters $s_{t'}^j$.

The resulting TPG encodes the precedence requirement that ensures the collision-free execution of the joint path $\pi$.
In \cite{honig2016multi}, a TPG is used to handle the kinematic constraints of the robots to ensure robust execution of the paths, while in this paper, we use TPG, combined with CBSS, as our first method to handle task duration in \abbrMATSPF, which is presented next.
	
	\section{\abbrCBSS}\label{matspf:sec:cbssstn}
	
% \subsection{Concepts and Overview}\label{sec:matspf:cbsstpg:1}
\subsection{Algorithm Description}

% Explain \abbrCBSS with mathematical notation.
% This section presents our first approach \abbrCBSS to solve \abbrMATSPF, which combines the existing CBSS and TPG in the following two steps.
% This approach does not provide solution optimality guarantee.

\abbrCBSS begins by ignoring all task duration (i.e., set all task duration to 0), which yields a \abbrMCPF problem. Then CBSS is applied to obtain a conflict-free joint path for this \abbrMCPF problem, denoted by $\pi_a = (\pi_a^1, \pi_a^2, ..., \pi_a^N)$.
Then, we propose \abbrTPGD to post-process the joint path $\pi_a$ by incorporating the task duration and avoid additional conflicts brought by the task duration, explained in Sec. \ref{sec:matspf:cbsstpg:2}.

\begin{remark}
After the first step in \abbrCBSS where a joint path is obtained, the remaining problem is how to incorporate the task duration into that joint path while avoiding conflicts.
A naive approach is, along the joint path $\pi_a$, let an agent $i$ stay at a target vertex $v$ until the task at $v$ is finished. Meanwhile, let all other agents wait at their current vertex until agent $i$ finishes the task, and then let all other agents start to move again.
This naive method would lead to many unnecessary wait actions.
We therefore develop \abbrTPGD, which is able to identify a small subset of agents that are affected by a task, and only let these affected agents wait until that task is finished.
\end{remark}

\subsection{TPG Post-Process}\label{sec:matspf:cbsstpg:2}
Our method \abbrTPGD (Alg. \ref{alg:tpg}) takes as input a joint path $\pi$ output by CBSS, task duration $\tau$ and a TPG $\mathcal{G}$ that is constructed by the method in~\cite{honig2016multi}, and returns conflict-free joint path $\pi_c$ with the task duration incorporated.
\abbrTPGD is based on the following property of TPG:
Recall that the directed edges in $\mathcal{G}$ indicates the precedence between events.
When an event happens (i.e., an agent has reached a location $s^i_t$), both $s^i_t$ and all out-going edges of $s^i_t$ can be deleted from $\mathcal{G}$ to remove the corresponding precedence constraints.
Then, an event has no precedence constraint (and can take place) if the corresponding vertex $s_t^i$ has zero in-degree in $\mathcal{G}$.

Based on this property, \abbrTPGD further introduces \textit{duration value} $D(s_t^i)$ for each vertex $s_t^i \in \mathcal{V}$ to handle task duration, where $D(s_t^i)$ represents the time agent $i$ should spend at location $s_t^i$.
Here, $D(s_t^i)$ includes both the waiting time (if any) and the task duration of agent $i$ at $s_t^i$.
% Let $deg_{in}(s_t^i)$ with $i \in I$ and $t \in [0,T]$ denotes the \textit{in-degree} of vertex $s_t^i \in \mathcal{V}$ in $\mathcal{G}$, and 
\abbrTPGD is an iterative algorithm, and in any iteration, let $L_0$ denote the list of vertices in $\mathcal{G}$ with zero in-degree, and let $L_d$ denote a list of vertices to be deleted in $\mathcal{G}$. 

% Except for calculating $D(v_t^i)$, the constructing process of TPG is very similar to \cite{honig2016multi}. Therefore, we recommend readers to read \cite{honig2016multi} for further detail of constructing TPG. Here, we only illustrate the extracting process of the conflict-free joint path $\pi_c$ (shown in Algorithm~\ref{alg:tpg}). Let $\mathcal{V}$ store the location of each agent in every iteration. 
At the beginning, \abbrTPGD initializes $\pi_c$ (the joint path to be returned) with $\emptyset$ and $v_c$ (the current joint vertex of all agents) with the starting vertices of all agents (Line~\ref{alg2:line3}-\ref{alg2:line4}).
% \abbrTPGD sets $k$ (the number of while-iterations) to be zero and 
\abbrTPGD initializes $D(s^i_t)$ for all $s^i_t \in \mathcal{G}$ by finding the waiting time and the task duration at $s^i_t$ based on $\pi_a$ and $\tau$.
If there is no wait or no task duration related to $s^i_t$, procedures WaitTime and TaskDuration simply return zero and $D(s^i_t)=0$.

In each iteration (Line~\ref{alg2:line6}-\ref{alg2:line19}), \abbrTPGD first finds the set $L_0$ of all vertices  with zero in-degree in $\mathcal{G}$ (Line~\ref{alg2:line7}).
Then, for each vertex $s^i_t$ in $L_0$, \abbrTPGD reduces its $D$-value by one unless $D(s^i_t)$ is already zero.
Then, \abbrTPGD finds all $s^i_t$ in $L_0$ that should be deleted from $\mathcal{G}$ (Line~\ref{alg2:line13}-\ref{alg2:line15}) via procedure CheckDelete and add them into $L_d$.
CheckDelete can be implemented in different ways~\cite{honig2016multi,honig2019persistent} and we consider the following three conditions when a vertex $s^i_t$ can be deleted: (i) the in-degree of $s^i_t$ is zero; (ii) $D(s^i_t)=0$ and (iii) the agent can reach the next vertex $s^i_{t+1}$ without colliding with any other agents.
Here, (iii) is needed since another agent $j$ may stay at vertex $s^j_{t'}=s^i_{t+1}$ for task execution, and if agent $i$ leaves $s^i_t$, $i$ cannot reach $s^i_{t+1}$.
Then, \abbrTPGD deletes all vertices in $L_d$ from $\mathcal{G}$ and append $v^i_c \in v_c$ (the current vertex of each agent) to the end of $\pi^i_c$ for each $i\in I$, if $i$ has not reached its goal yet.
In the next iteration, since vertices in $L_d$ are deleted from $\mathcal{G}$, $L_0$ will change and $v_c$ will be updated correspondingly.
\abbrTPGD terminates when $\mathcal{G}$ is empty, which means all agents have reached their goals (Line~\ref{alg2:line6}).

\begin{algorithm}[tb]
\caption{Pseudocode for \abbrTPGD}
\label{alg:tpg}
\begin{algorithmic}[1]
\small
    \State{\textbf{Input}: A TPG~$\mathcal{G}$, a joint path $\pi$ output by CBSS, and all task duration $\tau$.}
    \State{\textbf{Output}: A conflict-free joint path $\pi_c$ where each agent follows the same visiting order in $\mathcal{G}$}
    \State $\pi_c \gets (\emptyset, \emptyset,..., \emptyset)$ \label{alg2:line3}
    % \State{\small// $V_i$ stores the location of each agent in every iteration}
    \State $v_c \gets (v_0^0, v_0^1,..., v_0^N)$  \label{alg2:line4} %, $k\gets 0$
    \State $D(s^i_t) \gets $ WaitTime($v^i_t$, $\pi$) + TaskDuration($v^i_t$,$\pi$,$\tau$)
    \While{$\mathcal{G} \neq \emptyset$} \label{alg2:line6}
        \State{$L_0 \gets \textit{ZeroInDegVertices}(\mathcal{G})$} \label{alg2:line7}
        \State{$L_d \gets \emptyset$}
        \ForAll{$s_t^j \in L_0$}
        % \State $V_t[j] = v_t^j$
        \State $v_c^j \gets s^j_t$

        \If{$D(s_t^j) > 0$}
        \State{$D(s_t^j) = D(s_t^j) - 1$}
        \EndIf
        \EndFor

        \ForAll{$s_t^j \in L_0$} \label{alg2:line13}
        \If{\textit{CheckDelete($s_t^j$)}}
        \State Add $s_t^j$ to $L_d$
        \EndIf
        \EndFor \label{alg2:line15}
        
        % \ForAll{$v \in L_d$}
        \State Delete all $s \in L_d$ from $\mathcal{G}$ \label{alg2:line16}
        % \EndFor

        \ForAll{$i \gets 1~\textbf{to}~N$}
        \If{Agent $i$ has not reached the end of $\pi^i$}
        \State Append $v_c^i$ to the end of $\pi_c^i$
        \EndIf
        \EndFor
    \EndWhile \label{alg2:line19}
    \State \Return $\pi_c$ \Comment{Final conflict-free joint path}
\end{algorithmic}
\end{algorithm}

\begin{example}
For the example in Fig.~\ref{fig:toy_example}, after invoking CBSS, the resulting joint sequence is $\gamma^1=(8,9,10,11)$, $\gamma^2=(1,13)$, $\gamma^3=(2,6,14)$ and the resulting joint path is $\pi^1_a=(8,9,10,11)$, $\pi^2_a=(1,5,9,13)$, $\pi^3_a=(2,6,6,10,14))$.
If one simply add the task duration back to $\pi_a$, the resulting joint path $\pi^0_b=((8,9,9,9,10,10,11)$, $\pi^1_b=(1,5,9,13)$, $\pi^2_b=(2,6,6,6,6,6,6,10,14)$ is not conflict-free.
Fig.~\ref{fig:tpg} shows the constructed TPG by the method in~\cite{honig2016multi} and the blue number on the upper right corner of each circle in Fig.~\ref{fig:tpg} shows the $D$-value introduced by our \abbrTPGD.
The returned joint path by \abbrTPGD is $\pi_c^0=(8,9,9,9,10,10,11)$, $\pi_c^1=(1,5,5,5,9,13)$, $\pi_c^2=(2,6,6,6,6,6,6,10,14))$, which is conflict-free.
\end{example}

\begin{figure}[tb]
    \centering
    \includegraphics[width=.9\linewidth]{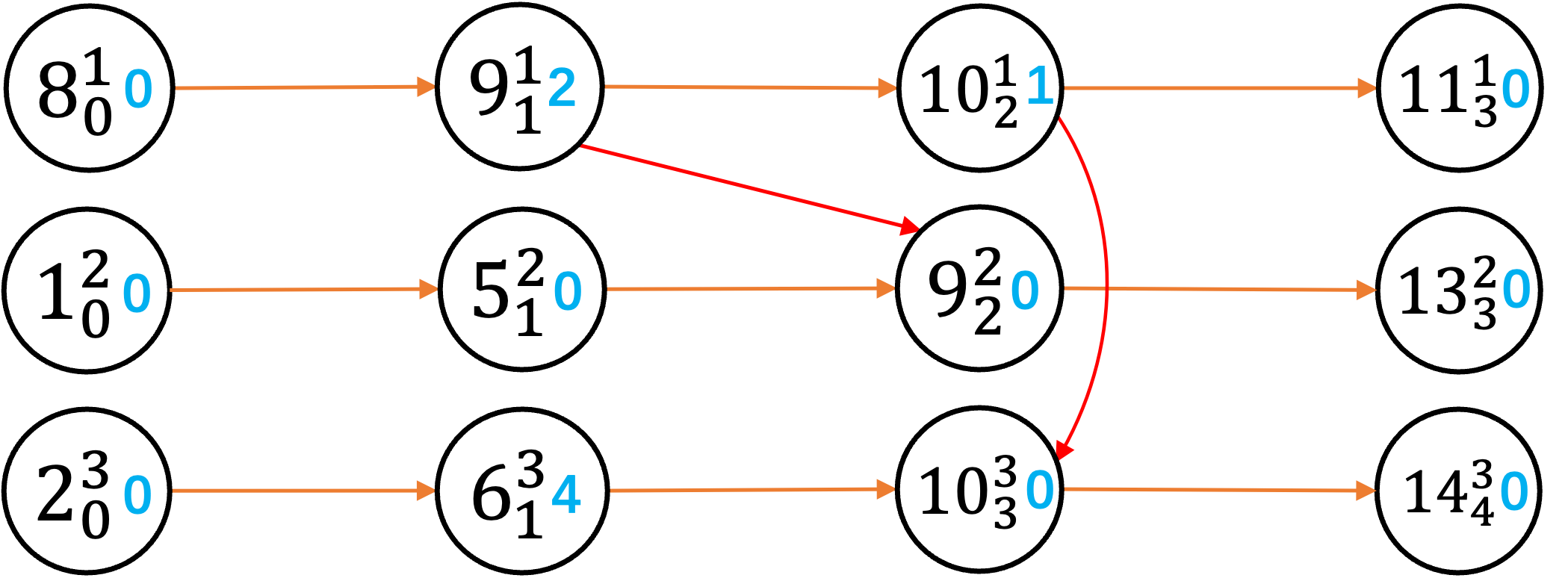}
    \caption{TPG of the example in Fig.~\ref{fig:toy_example}. Vertex $s_t^i$ means agent $i$ moves to location $s$ at time $t$, and the number in blue is the value of $D(s_t^i)$. The orange arrows represent Type 1 edges and the red arrows represent Type 2 edges.}
    \label{fig:tpg}
    \vspace{-0.4cm}
\end{figure}

\subsection{Properties}

% We now show the prove that the joint path generated by \abbrTPGD is guaranteed to be conflict-free.
We say a joint path $\pi_1$ follow the same visiting order as $\pi_2$, if the following conditions hold: 
(i) every vertex that appears in $\pi_1^i$ also appear in $\pi_2^i$ for all $i\in I$; 
(ii) if vertex $u$ appears after $v$ in $\pi_1^i$, then $u$ also appears after $v$ in $\pi_2^i$ for all $i\in I$; 
(iii) if a vertex $v$ appears in two agents' paths $\pi_1^i,\pi_1^j, i\neq j$ and $j$ visits $v$ after $i$ has visited $v$, then $j$ also visits $v$ after $i$ has visited $v$ in $\pi_2^i,\pi_2^j$. Note that two joint paths with the same visiting order may differ due to variations in the waiting periods experienced by some agents along each path. 

The construction of the TPG $\mathcal{G}$ captures the precedence between the agents in $\pi$ using Type 1 and Type 2 edges \cite{honig2016multi}.
The $\pi_c$ returned by \abbrTPGD preserves the precedence in $\pi$ by iteratively deleting and appending the zero in-degree vertices (via \textit{CheckDelete}) to the returned joint path.
By enforcing the precedence, $\pi_c$ follows the same visiting order as $\pi$.
We summarize this property in the following theorem.
\begin{theorem}
\label{theo:order}
Let $\pi$, $\tau$ and $\mathcal{G}$ denote the input to \abbrTPGD. Then, $\pi_c$ returned by \abbrTPGD follows the same visiting order as $\pi$.
\end{theorem}

Theorem~\ref{theo:order} ensures that all agents eventually reach their respective goals along $\pi_c$.
Additionally, the third condition in \textit{CheckDelete} in \abbrTPGD ensures that an agent does not move to its next vertex along its path unless the agent can reach it without conflict.
Therefore, we have the following theorem.

\begin{theorem}
\label{theo:tpgd-conflict-free}
The joint path $\pi_c$ generated by \abbrTPGD is conflict-free solution to \abbrMATSPF.
\end{theorem}

\abbrCBSS can't guarantee the optimality of the returned solution $\pi_c$.
% because we ignore task duration when planning target sequences and TPG post-process has no optimality guaranteed as well.
Next, we will introduce our second method \abbrCBXS which finds an optimal solution for \abbrMATSPF.

    \section{\abbrCBXS}\label{matspf:sec:cbxs}
	
% \subsection{Concepts and Overview}\label{sec:matspf:cbxs}
% Explain \abbrCBXS with mathematical notation.

% Our second method, Conflict-Based Steiner Search with Task Duration (\abbrCBXS), is similar to CBSS, which interleaves target sequencing and path planning.
% To solve \abbrMATSPF to optimality, \abbrCBXS needs to consider task duration during planning, and 
There are two main differences between \abbrCBXS and CBSS: First, \abbrCBXS modifies the transformation method for target sequencing to handle task duration (Sec.~\ref{subsec:transformation});
Second, when a conflict is detected, \abbrCBXS introduces a new branching rule, which leverages the knowledge of task duration, to resolve the conflict more efficiently than using the basic branching rule in CBSS (Sec.~\ref{subsec:branchingrules}).

% \subsection{Algorithm Description}

% \abbrCBXS is shown in Algorithm~\ref{alg:CBSS}, where the differences between \abbrCBXS and CBSS are highlighted in blue. On Line~4, \textit{K-best-Sequencing} solves a K-best TSP to find a set of K cheapest joint sequences, where a new transformation method is used to consider task duration. On Line~14, \textit{GenerateConstraints} uses a new branching rule which adds multiple constraints to the agents involved in a conflict as opposed to one constraint at a time, order to improve resolution efficiency.

\subsection{Transformation for Sequencing}
\label{subsec:transformation}
% Explain how we adjust the original transformation stage in Conflict-Based Steiner Search to generate optimal target sequence with task duration. List the four types of edges in the cost map to convert mTSP problem to a TSP problem.

As aforementioned,
% in Sec.~\ref{cbssd:sec:method:K-best-seq},
to find K-best joint sequences, CBSS first leverages a transformation method to convert a mTSP to an equivalent (single agent) TSP, then uses a partition method to solve a K-best (single agent) TSP, and then un-transform each of the obtained K-best (single agent) tours to a joint sequence.
Here, \abbrCBXS only modifies the transformation method while the partition method remains the same.

% To generate K-best joint sequence for a \abbrMATSPF problem, \abbrCBXS leverages a transformation method similar to \cite{VANDERPOORT1999409} and \cite{ren2023cbss} to convert the target graph $G_T$ into a directed \textit{transformed graph} $G_{TF} = (V_{TF}, E_{TF}, C_{TF})$ (subscript T F stands for “transformation”).
The transformation method first creates a transformed graph $G_{TF}$ out of the target graph $G_T$ by defining the vertex set $V_{TF}$ as $V_{TF} := V_0 \bigcup U$, where $U$ is an augmented set of targets and goals: for each $v \in V_t \bigcup V_d$, a copy of $v^i$ of $v$ for agent $i$ is created if $i \in f_A(v)$.
Making these copies of each target $v$ for agents in $f_A(v)$ help handle the assignment constraints~\cite{ren2023cbss}. 
The edges $E_{TF}$ and their corresponding costs $C_{TF}$ are set in a way such that an optimal solution tour in $G_{TF}$ has the features that allow the reconstruction of a joint sequence out of the tour.
In short, (i) when an optimal tour visits the goal of agent $i$, the tour will visit the initial vertex of the next agent $i+1$.
This allows breaking the tour into multiple sub-tours based on the goals during the reconstruction (Fig.~\ref{fig:transform}(b)), and each sub-tour becomes a target sequence of an agent in the resulting joint sequence.
(ii) When an optimal solution tour in $G_{TF}$ visits a copy of targets or destination $v$, the tour will immediately visit all other copies of $v$ before visiting a next target.
When reconstructing the joint sequence, only the first copy of a target is kept, which identifies the agent that is assigned to this target, and all other subsequent copies are removed during the reconstruction of a joint sequence (Fig.~\ref{fig:transform}(c)).
% The four types of edges in \abbrCBXS is still the same as in CBSS. The details about defining above $V_{TF}$ and $E_{TF}$ can be found in \cite{ren2023cbss}.
% Differing from \cite{VANDERPOORT1999409} and \cite{ren2023cbss}, \abbrMATSPF defines edge costs $C_{TF}$ based on not only the assignment constraints $f_A$ but also the task duration $\tau$. 

For each target or goal $v \in V_t \bigcup V_d$, a copy of $v^i$ of $v$ for agent $i$ is created if $i \in f_A(v)$.
To include task duration, the new part in \abbrCBXS is that, it adds an additional cost $\tau^i(v)$ for the in-coming edge that goes into $v^i$, where $\tau^i(v)$ is the task duration.
Note that although the task duration is defined on vertices, we are able to add the task duration to a corresponding edge because there is a one-one correspondence between each target and an in-coming edge.
For the same target, different agents may have different duration, which is also allowed here, since for every target, there is a unique in-coming edge for each agent.
% The intuition of this modification is that the additional cost $\tau^i(v)$ represents the time (cost) that agent $i$ spends on its assigned target $v$ to execute the corresponding task.
The property of the transformation method in \abbrCBXS is summarized with the following theorem.
The proof in \cite{5160666} and \cite{ren2023cbss} can be easily adapted.

\begin{theorem}\label{thm:transformation}
For a \abbrMATSPF Problem, given $G_T$, $f_A$ and $\tau$, the transformation method computes a minimum cost joint sequence that visits all targets and ends at goals while satisfying all assignment constraints.
\end{theorem}

% \begin{figure}[tb]
%     \centering
%     \subfloat[]
%     % \subfigure[Transformed Graph]
%     {
%         % \label{transform0}
%         \begin{minipage}[b]{.9\linewidth}
%             \centering
%             \includegraphics[scale=0.35]{source/figures/transform0.png}
%         \end{minipage}
%     }
%     \hfill
%     \subfloat[]
%     % \subfigure[TSP Tour]
%     {
%         % \label{transform1}
%         \begin{minipage}[b]{.45\linewidth}
%             \centering
%             \includegraphics[scale=0.27]{source/figures/transform1.png}
%         \end{minipage}
%     }
%     \subfloat[]
%     % \subfigure[Joint Sequence]
%     {
%         % \label{transform2}
%         \begin{minipage}[b]{.45\linewidth}
%             \centering
%             \includegraphics[scale=0.27]{source/figures/transform2.png}
%         \end{minipage}
%     }
%     \caption{The transformation method for the toy example in Fig.~\ref{fig:toy_example}. Here, (a) shows the transformed graph $G_{TF}$. The cost related to task duration is denoted by $+x$, where $x$ is the corresponding task duration. (b) shows the TSP tour in $G_{TF}$ computed by a TSP solver. (c) shows the joint sequence untransformed from the TSP tour.}
%     \label{fig:transform}
% \end{figure}

\begin{figure}[t]
\centering
    \includegraphics[width=0.99\linewidth]{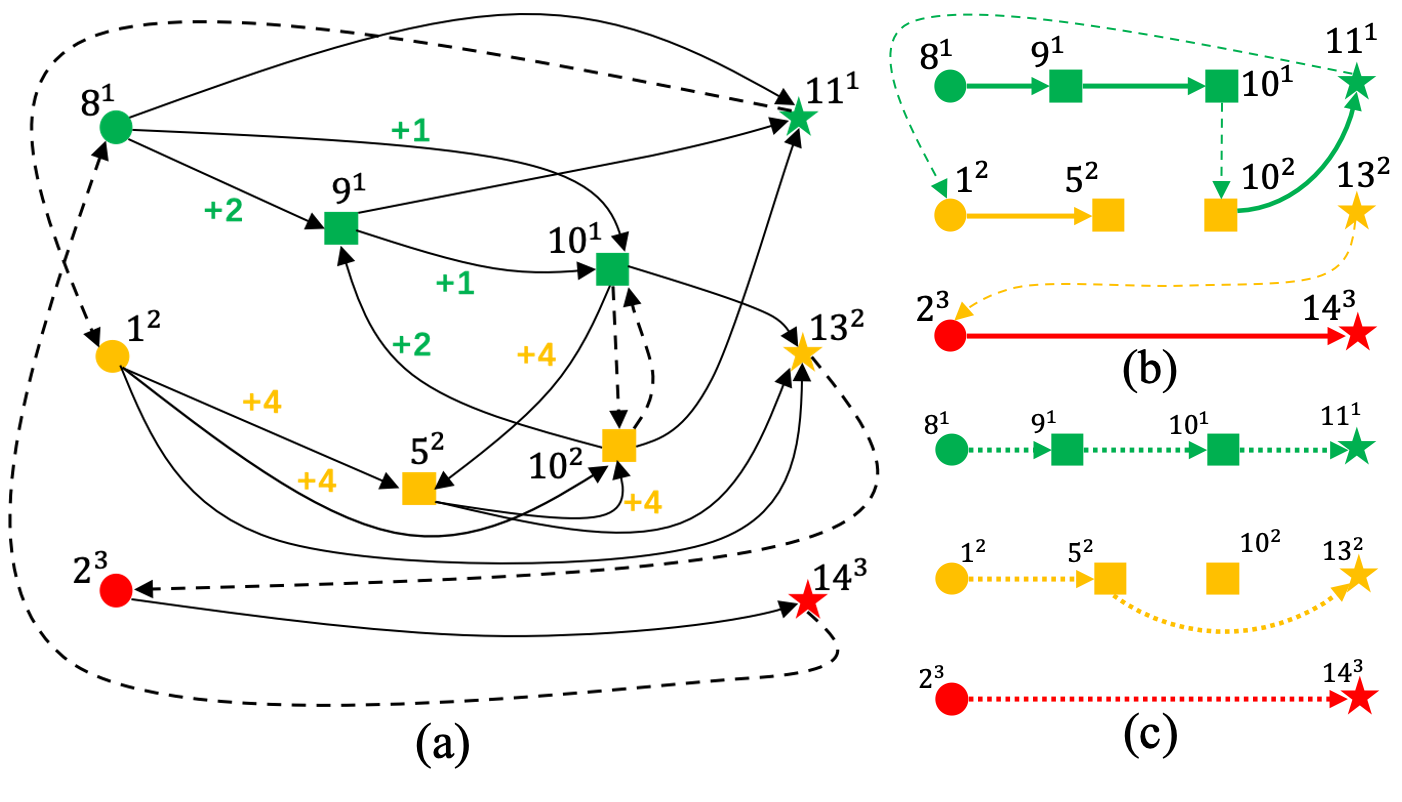}
    \vspace{-0.8cm}
    \caption{The transformation method for the toy example in Fig.~\ref{fig:toy_example}. Here, (a) shows the transformed graph $G_{TF}$. The cost related to task duration is denoted by $+x$, where $x$ is the corresponding task duration. (b) shows the TSP tour in $G_{TF}$ computed by a TSP solver. (c) shows the joint sequence untransformed from the TSP tour.}
    \vspace{-0.4cm}
    \label{fig:transform}
\end{figure}

\begin{example}
For the problem in Fig.~\ref{fig:toy_example}, as shown in Fig.~\ref{fig:transform}, the generated joint sequence by our transformation method is $\gamma = \{(8,9,10,11),(1,13),(2,6,14)\}$.
\end{example}

\subsection{Different Branching Rules} 
\label{subsec:branchingrules}
% Motivated by Multi-Constraint CBS(MO-CBS) in \cite{li2019multi}, explain the new branching rule we discussed in WeChat with mathematical notation.
Naively applying the conflict resolution method in CBSS to \abbrMATSPF can lead to inefficient computation.
% Motivated by Multi-Constraint CBS(MO-CBS) in \cite{li2019multi}, we design a new branching rule for \abbrMATSPF. Previously 
% In CBSS, to resolve a conflict, two constraints are generated as aforementioned in Sec.~\ref{subsec:CBSS}.
% However, due to the task duration, the same conflict may be detected and the same constraints are generated which will cause the redundancy of the generated search tree. 
For example: when a vertex conflict $(i,j,v,t)$ is detected in an iteration of CBSS, two constraints $(i,v,t)$ and $(j,v,t)$ are created and leads to two new branches.
Say $i$ executes the task at target $v$ during time $[t,t+\tau^i(v)]$.
Then, the branch with the constraint $(j,v,t)$ replans the path for agent $j$.
Along the replanned path, agent $j$ may enter $v$ at time $t+1$, which is still in conflict with agent $i$ if $t+1 \in [t,t+\tau^i(v)]$.
This branch thus requires another iteration of conflict resolution.
When $\tau^i(v)$ is large, it can lead to many iterations of conflict resolution and thus slow down the computation.
% However, in the next $\tau^i(t)$, the conflict happening between agent $i$ and $j$ in target $v$ at time $t+1,...,t+\tau^i(t)$ will be continually detected and the branch with constraint $(j,v,t'),t'=t+1,...,t+\tau^i(v)$ will be generated as well. Thus, to deal with this constraint repetition of the same agent in the same target, we propose a new branching rule that can generate multiple constraints for one conflict.
Our new branching rule seeks to resolve multiple conflicts related to the same task in one iteration and therefore saves computational effort.

\begin{algorithm}[tb]
\caption{Pseudocode for GenerateConstraints}
\label{alg:constraints}
\small
\begin{algorithmic}[1]
    \State{\textbf{Input}: A vertex or edge conflict $\textit{cft}$}
    \State{\textbf{Output}: The generated constraints $\Omega$}
    \State $\Omega \gets \emptyset;\omega_1 \gets \emptyset;\omega_2 \gets \emptyset$
    \If{\textit{cft} is an edge conflict}
    \State $\omega_1 = \{(i,e,t)\}$; $\omega_2 = \{(j,e,t)\}$
    % \State $\omega_2 = \{(j,e,t)\}$
    \ElsIf{\textit{CheckExecution}($i$)} \label{alg3:line6}
    \State $t_s, t_e = \textit{TaskStartEnd}(i,v,t)$
    \State $\omega_1 = \{(i,v,t_1)|t_s\leq t_1\leq t\}$
    \State $\omega_2 = \{(i,j,t_2)|t\leq t_2\leq t_e\}$ \label{alg3:line9}
    \ElsIf{\textit{CheckExecution}($j$)} \label{alg3:line10}
    \State $t_s, t_e = \textit{TaskStartEnd}(j,v,t)$
    \State $\omega_1 = \{(i,v,t_1)|t\leq t_1\leq t_e\}$
    \State $\omega_2 = \{(i,v,t_2)|t_s\leq t_2\leq t\}$ \label{alg3:line13}
    \Else \label{alg3:line14}
    \State $\omega_1 = \{(i,v,t)\}$; $\omega_2 = \{(j,v,t)\}$ \label{alg3:line15}
    % \State $\omega_2 = \{(j,v,t)\}$
    \EndIf
    \State $\Omega \gets \{\omega_1, \omega_2\}$
    \State \Return $\Omega$ 
\end{algorithmic}
\end{algorithm}

We modify the function \textit{GenerateConstraints} on Line12 in Alg.~\ref{alg:CBSS}, and the modified branching rule is shown in Alg.~\ref{alg:constraints}.
For an edge conflict $(i,j,e,t)$, agent $i$ and $j$ must not be executing any task, so \abbrCBXS resolves the conflict in the same way as CBSS by generating two constraints $(i,e,t)$ and $(j,e,t)$. 
For a vertex conflict $(i,j,v,t)$, agent $i$ and $j$ can not be simultaneously executing tasks at target $v$ at time $t$. Therefore, there are three cases to be considered:
\begin{itemize}
    \item 
First, if \textit{CheckExecution} finds agent $i$ is executing task at vertex $v$, two \emph{sets of constraints} are generated. The first set is $\{(i, v, t_1)|t_s \leq t_1 \leq t\}$ where $t_s$ is the time that agent $i$ starts the task at $v$. The second set of constraints is $\{(j, v, t_2)|t \leq t_2 \leq t_e\}$ where $t_e$ is the time that agent $i$ ends the task at $v$ (Line~\ref{alg3:line6}-\ref{alg3:line9}).
    \item 
Second, if \textit{CheckExecution} finds agent $j$ is executing task at $v$, two sets of constraints are generated in a similar way (Line~\ref{alg3:line10}-\ref{alg3:line13}).
    \item
If none of the agents are executing task at $v$, two constraints $(i,v,t)$ and $(j,v,t)$ are generated in the same way as in CBSS (Line~\ref{alg3:line14}-\ref{alg3:line15}).
\end{itemize}

We discuss the idea behind this branching rule in Sec. \ref{subsec:cbssd_analysis}.

\begin{example}
In Fig.~\ref{fig:toy_example}, there is a vertex conflict $(i=1, j=2, v=9, t=1)$.
When using the branching rule in CBSS to resolve this conflict, additional vertex conflicts $(i=1,j=2,v=9,t=2),(i=1,j=2,v=9,t=3)$ will be detected and resolved in the subsequent iterations.
When using our new branching rule, the conflict $(i=1, j=2, v=9, t=1)$ can be resolved in only one iteration. 
\end{example}

\begin{example}
Here we provide a comparison between the solutions returned by \abbrCBSS and \abbrCBXS for the toy problem in Fig.~\ref{fig:toy_example}.
% Furthermore, this new branching rule can be proved to be complete \cite{li2019multi} which helps to ensure the optimality of \abbrCBXS (Sec.~\ref{ana:cbxs}). For instance, if we apply \abbrCBSS to solve the toy example in Fig.~\ref{fig:toy_example}, due to Sec. \ref{sec:matspf:cbsstpg:2}, we will get solution 
\abbrCBSS returns $\pi_c^0=((8, 9, 9, 9, 10, 10, 11)$, $\pi_c^1=(1, 5, 5, 5, 9, 13)$, $\pi_c^2=(2, 6, 6, 6, 6, 6, 6, 10, 14))$ with the cost of 19. 
\abbrCBXS returns $\pi_*^0=((8, 9, 9, 9, 10, 10, 11)$, $\pi_*^1=(1, 5, 5, 5, 9, 13)$, $\pi_*^2=(2, 6, 6, 6, 6, 6, 10, 14))$ with the cost of 18 which is the optimal cost.
One can expect larger cost difference when the task duration becomes larger.
\end{example}

\subsection{Solution Optimality of \abbrCBXS}\label{subsec:cbssd_analysis}

We first present the property of our new branching rule and then use it to show the solution optimality of \abbrCBXS.

\begin{definition}[\textit{Mutually Disjunctive Constraints}]
Two vertex or edge constraints for agents $i$ and $j$ are mutually disjunctive~\cite{li2019multi}, if there does not exist a conflict-free joint path such that both constraints are violated.
% In particular, the constraints that CBS adds to two child nodes are always mutually disjunctive~\cite{li2019multi}.
Besides, two sets of constraints corresponding to agent $i$ and $j$ are mutually disjunctive if every constraint in one set is mutually disjunctive with every constraint in the other set.
\end{definition}

\begin{theorem}
\label{lemma:md}
For any conflict, the branching rule in Sec.~\ref{subsec:branchingrules} generates two mutually disjunctive constraint sets $C_1$ and $C_2$.
\end{theorem}
\begin{proof}
If the conflict between $i$ and $j$ is an edge or vertex conflict and none of the agents execute a task in $v$, the constraints are generated in the same way as CBS and are thus mutually disjunctive~\cite{li2019multi}.
Otherwise, consider the two sets $C_1,C_2$ of constraints generated in Alg.~\ref{alg:constraints}.
We prove that $C_1$ and $C_2$ are mutually disjunctive by contradiction.
% $C_1=\{(i,v,t_1)|t_s\leq t_1\leq t\}$ and $C_2=\{(i,v,t_2)|t \leq t_2\leq t_e\}$. It's obvious that any constraint in $C_1$ is mutually disjunctive with any constraint in $C_2$. 
Assuming that $C_1$ and $C_2$ are not mutually disjunctive, then there must exist at least one conflict-free joint path $\pi'$ that violates at least one constraint in $C_1$ and at least one constraint in $C_2$. 
% If not, we assume there is one pair of non mutually disjunctive constraints $\omega_1 \in C_1$ and $w_2 \in C_2$. Thus, there is one pair of conflict-free paths for $a_i$ and $a_j$ that one violates $\omega_1$ and the other violates $\omega_2$. In this way, 
Then, $i$ and $j$ are in conflict along $\pi'$, which leads to contradiction.
% in vertex $v$ which is contradictory to our assumption. 
% We prove by contradiction.
% Assuming that $C_1$ and $C_2$ are not mutually disjunctive, then there must exist at least one conflict-free joint path $\pi'$ that violates at least one constraint in $C_1$ and at least one constraint in $C_2$. Then, the conflict $cft$ can not be resolved in $\pi'$ through generating constraint set $C_1$ and $C_2$ which is contradictory.
\end{proof}

% \begin{theorem}
% \label{thm:complete_cbxs}
% % In the high-level search of \abbrCBXS, given a joint sequence $\gamma$ from the transformation method, we can find the complete set of conflict-free joint paths by resolving the conflicts along the $\gamma$.
% In the high-level search of \abbrCBXS, given a joint sequence $\gamma$ from the transformation method, our search will not lose any conflict-free joint paths by resolving the conflicts along the joint paths generated by the joint sequence $\gamma$.
% \end{theorem}

% \begin{proof}
% If \textbf{Theorem~\ref{thm:complete_cbxs}} does not hold, there must be one conflict-free joint path $\pi'$ which is not included in the search tree expanded from $\gamma$. It indicates that $\pi'$ violates at least one pair of two sets of constraints. However, because all pairs of constraints we generate are \textit{mutually disjunctive} (as proved in \textbf{Lemma~\ref{lemma:md}}), $\pi'$ must have at least one $cft$, contradicting its collision-free characteristic. Thus, the new branching process of the search tree in Sec..~\ref{subsec:branchingrules} does not lose any set of conflict-free paths, and we say it is complete.
% \end{proof}

With Theorem~\ref{thm:transformation} \ref{lemma:md}, the proof in~\cite{ren2023cbss} can be readily adapted to show that the solution returned by \abbrCBXS is optimal for \abbrMATSPF.
First, the transformation method in Sec.~\ref{subsec:transformation} (Theorem~\ref{thm:transformation}) and the partition method in~\cite{ren2023cbss,VANDERPOORT1999409,Hamacher1985KBS} finds K-best joint sequences for a given $K$.
For each of those joint sequences, \abbrCBXS uses CBS-like search to resolve conflicts, and Theorem \ref{lemma:md} ensures that the search with the new branching rule can find a conflict-free joint path if one exists.
Finally, \abbrCBXS is same as CBSS~\cite{ren2023cbss} when generating new joint sequences and resolving conflicts, by using an OPEN list to prioritize all candidate nodes based on their $g$-costs and always select the minimum-cost one for processing.
The returned solution is thus guaranteed to be a conflict-free joint path with the minimum-cost.
We thus have the following theorem. 

\begin{theorem}
\label{thm:optimal_cbxs}
For a solvable \abbrMATSPF problem (i.e., the problem has at least one conflict-free joint path), \abbrCBXS terminates in finite time and returns an optimal solution.
\end{theorem}

% \begin{proof}
% According to \cite{li2019multi}, \abbrCBXS will terminates if there exist a set of conflict-free paths. Therefore, the rest part of the theorem is to show that the first solution returned by \abbrCBXS when it terminates is optimal. The proof in \cite{ren2023cbss} and \cite{sharon2015conflict} can be adapted here to prove the optimality. According to \cite{ren2023cbss}, for a \abbrMCPF problem, if we set $\epsilon=0$, CBSS can return an optimal solution with minimal cost. That's because the optimality is guaranteed both in the K-best sequencing process and in the high-level search of CBSS, which also holds in \abbrCBXS: 1) In the K-best sequencing process, according to \textbf{Theorem~\ref{thm:transformation}}, the generated joint sequence has the minimal cost. 2) Given the optimal joint sequence, the search tree is expanded in a best-first manner (i.e. always select the node with the minimal cost while satisfying its constraint set so far) and the expansion does not lose any conflict-free path as proved in \textbf{Theorem~\ref{thm:complete_cbxs}}. Thus the first solution returned is optimal (or bounded suboptimal if $\epsilon \neq 0$ in CBSS).
% \end{proof}
	
	% \section{Analysis}\label{matspf:sec:analysis}
	% \input{source/analysis}
	
	\section{Results}\label{matspf:sec:result}
	
\subsection{Test Settings}\label{matspf:sec:result:settings}
We implement both \abbrCBSS and \abbrCBXS in Python and use LKH-2.0.10 as the TSP solver.\footnote{LKH \cite{lkh} (\url{http://akira.ruc.dk/~keld/research/LKH/}) is a heuristic algorithm for TSP, which does not guarantee solution optimality but often finds an optimal solution for large-scale TSP instances in practice. We use LKH due to its computational efficiency. Other TSP solvers can also be used.}
% Note that, since the tour returned by LKH is not guaranteed to have the minimum cost, the resulting implementation of \abbrCBXS (with $\epsilon$ = 0) is not guaranteed to return an optimal solution.}
Similarly to \cite{ren2023cbss}, we use SIPP~\cite{phillips2011sipp} as the low-level
planner in a space-time graph $G \times \{0, 1, 2, ..., T \}$ subject to vertex and edge constraints.
We select two maps of different sizes from a online data set\footnote{\url{https://movingai.com/benchmarks/mapf/index.html}} and generate a four-connected grid-like graph $G$ from each of the maps.
This data set also includes 25 test instances for each map, where each instance includes hundreds of start-goal pairs. For each case, we use the first start-goal pair as the $v_o$ and $v_d$ for the robot and select randomly from the rest of the goals as the target $v_t$.
All tests run on a computer with an Intel Core i9-13900K CPU and 64 GB RAM. 
Each test instance has a runtime limit of 1 minute and $\epsilon$, a hyper-parameter in CBSS~\cite{ren2023cbss} that allows for bounded sub-optimal solutions, is always set to be zero.
% For the rest of this article, let $N$ denote the number of agents, $M$ denote the number of targets and $\tau^i(v)$ denote the task duration that agent $i$ is assigned to spend at target $v$. The number of destinations are not included in $M$.
Recall that, $N,M$ are the number of agents and targets respectively and the number of destinations are not included in $M$.

Since there many variations of \abbrMATSPF by selecting different $f_A$ and $\tau$, it is impossible to evaluate all of them.
We select three representative scenarios: 
\begin{itemize}
    \item (Scene 1) All targets are \abbranony (i.e., $f_A(v) = I, \forall v \in V_t \bigcup V_d\ and\ \tau^i(v) = \tau^j(v), \forall i,j \in f_A(v)$).  Here, the transformation in \abbrCBXS can be simplified in the same way as discussed in \cite{5663700,ren2021ms,ren2023cbss}: There is no need to make copies of targets and destinations for each agent and the related edges can be deleted.

    \item (Scene 2) Every destination is assigned to a unique agent. Every target has two randomly chosen eligible agents, and the task duration of all targets stays the same for all eligible agents.

    \item (Scene 3) Similar to Scene 2, every destination is assigned to a unique agent and every target has two randomly chosen eligible agents. Additionally, every target has heterogeneous task duration for the eligible agents, which is a randomly sampled integer. The sampling ranges vary in different tests and will be elaborated later.
\end{itemize}

% In order to ensure internal consistency of task duration and thus reflect the overall impact of increased task duration on algorithm performance, the tests in Sec.~\ref{subsec:result1} and Sec.~\ref{subsec:result2} only apply Scene 1 and Scene 2. Scene 3 with heterogeneous task duration will be validated in the Gazebo Simulation in Sec.~\ref{subsec:gazebo}.
% {\red I reached here. The content below is just my draft.}

\subsection{\abbrCBXS Versus \abbrCBSS}
\label{subsec:result1}
This section compares \abbrCBXS and \abbrCBSS in Scene 1 and 2 for $N \in \{10\}, M \in \{10,20,30,40,50\}$ and $\tau^i(v) \in \{2,5,10,20\}$. We test with two maps, Random $32\times 32$ and Maze $32 \times 32$.
Fig.~\ref{fig:SR} and~\ref{fig:CR} report the corresponding success rates and cost ratios. The cost ratio is defined as:
\begin{equation}
cost\ ratio = \frac{cost(\pi_{\abbrCBSS})-cost(\pi_{\abbrCBXS})}{cost(\pi_{\abbrCBSS})}\times100\%
\end{equation}

% Problem1: MSMP+Task Duration

% Problem2: MCPF+Task Duration

% Test: multiple maps like Room 32x32, Maze 32x32, Random 32x32, N=10, M=10,20,30,40,50, Duration=2,5,10,20.

\begin{figure}[tb]
    \centering
    \includegraphics[width=1\linewidth]{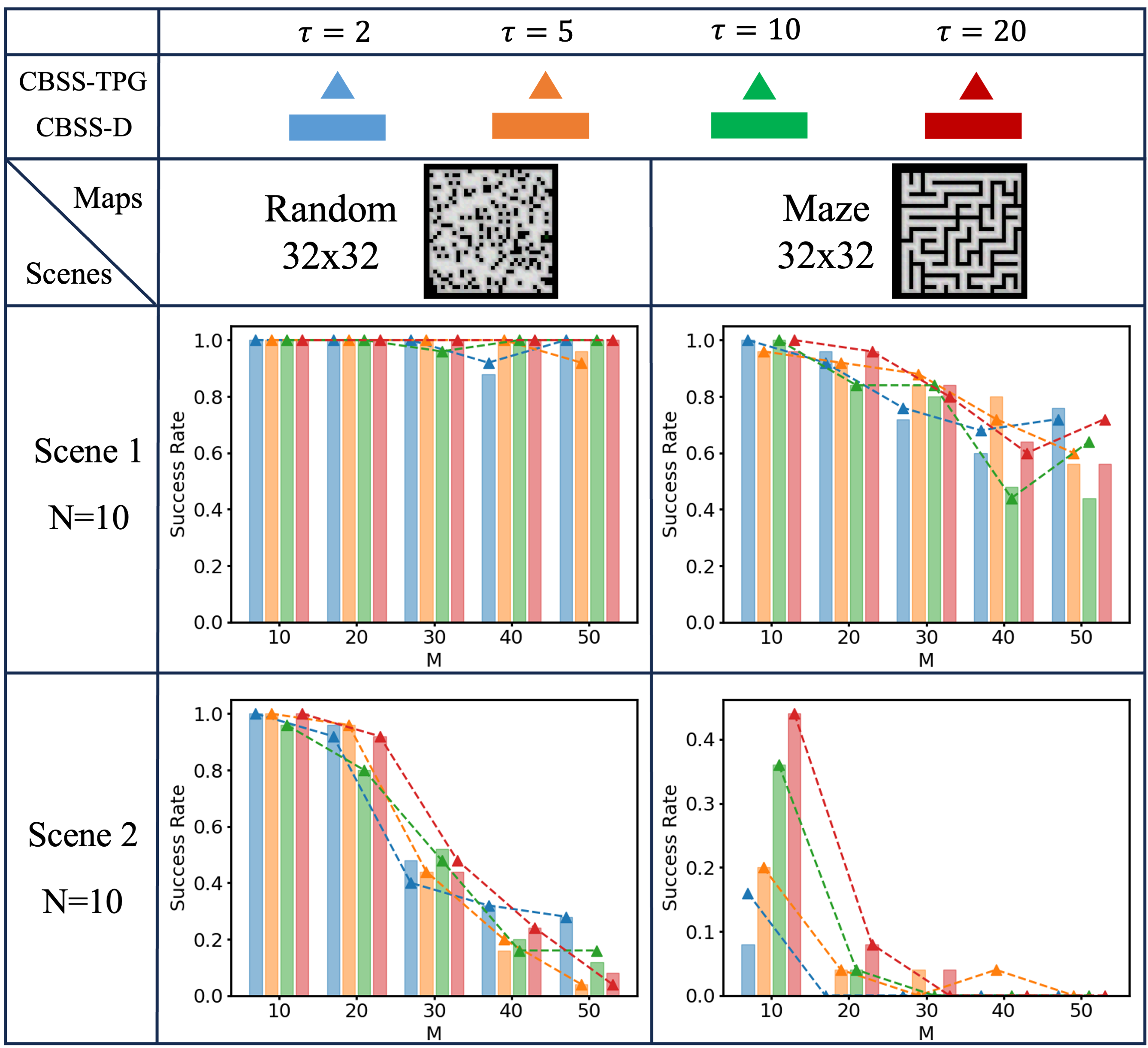}
    \caption{
    The success rates of \abbrCBSS and \abbrCBXS with various number of targets $M$ and task duration $\tau$.
    % when $N=10, M \in \{10,20,30,40,50\}$ and $\tau^i(v) \in {2,5,10,20}$. Different color indicates the different task duration and the $x$ axis represents the number of targets (M). Within time limit (60 seconds), \abbrCBXS and \abbrCBSS shows negligible difference on the performance of success rate.
    The maze is harder than the random map in general and both algorithms achieve similar success rates.
    }
    \label{fig:SR}
\end{figure}
% \begin{enumerate}
%     \item \textit{Success Rates}: 
For the success rates, as shown in Fig.~\ref{fig:SR}, first, the maze map is harder than the random map in general, where the agents are more likely to run into conflict with each other.
Both planners have lower success rates in the maze map than the random map.
Second, as $M$ increases, the corresponding TSP is harder to solve and thus the success rates decrease. 
% \abbrCBXS and \abbrCBSS are similar in general. 
    % This is because solving a TSP problem is relatively more computational~\cite{ren2023cbss} compared to the process of conflict-based search in \abbrCBSS and \abbrCBXS. 
Most of the failed instances time out when solving a TSP problem. 
% Since changing the edge cost in the transformation method in \abbrCBXS can not explicitly change the complexity of the TSP problem, it is reasonable to see the difference of success rate between \abbrCBXS and \abbrCBSS is negligible. 
In addition, \abbrCBXS and \abbrCBSS have similar success rate in general.

% We then explore the difference in runtime between \abbrCBXS and \abbrCBSS. Among the instances where \abbrCBXS and \abbrCBSS both succeed, both planners have similar runtime (i.e., difference is less than one second) for 15\% of them. Furthermore, \abbrCBSS is faster than \abbrCBXS for 75\% of the instances, while \abbrCBXS is faster than \abbrCBSS for 10\% of the instances. It indicates that \abbrCBSS tends to run faster than \abbrCBXS on average.

    % \item \textit{Solution Cost Ratio}:
For the cost ratios shown in Fig.~\ref{fig:CR}, \abbrCBXS finds better (up to 40\% cheaper) solutions than \abbrCBSS does, especially when $\tau$ is large.
% Besides, as $M$ and $\tau$ increases, the problem becomes more difficult and the map becomes easy for agents to collide (i.e. from Random to Maze), the cost ratio would increase. For example, especially in the Scene 2 and Maze map, the increasing trend of cost ratio is obviously shown when $M=10$. It indicates that when task duration increases, the cost ratio increases. The principle behind this phenomenon is that when agents are more likely to collide in the target with task duration and the task duration becomes larger, the optimality guaranteed in \abbrCBXS through considering task duration when planning joint sequences will save more cost compared to \abbrCBSS.
This is expected since \abbrCBSS does not consider task duration in planning and simply let the related agents wait till the other agents finish their task, while \abbrCBXS considers task duration during planning.
% \end{enumerate}

% Except for the numerical differences above, \abbrCBSS is more generative than \abbrCBXS due to the \abbrTPGD module which can be combined with other MCPF planner to eliminate conflicts brought by additional task assignments (e.g the task duration of \abbrMATSPF in this article).

\begin{figure}[tb]
    \centering
    \includegraphics[width=1\linewidth]{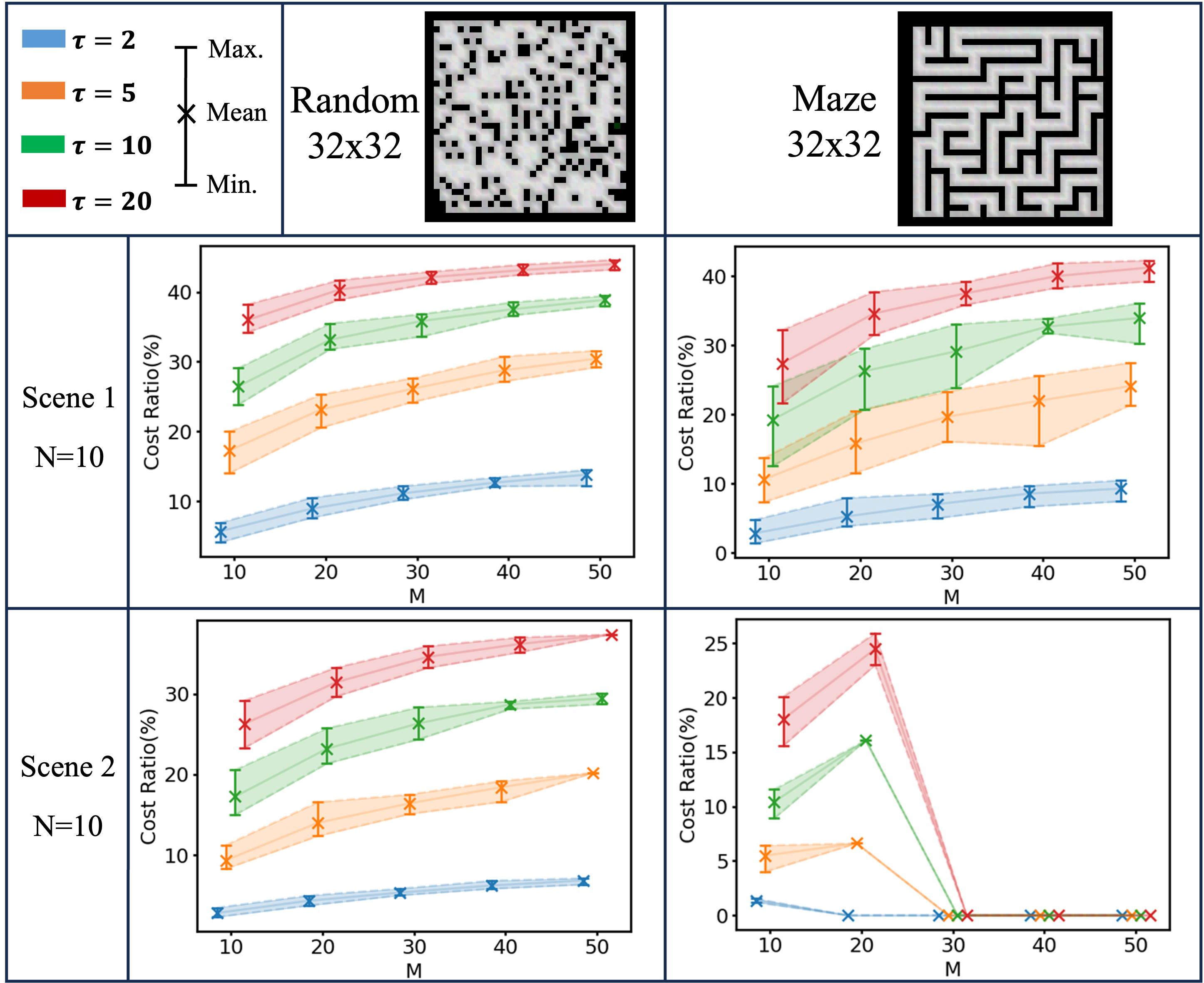}
    \caption{The cost ratios of \abbrCBSS and \abbrCBXS. The vertical axis shows the percentage.
    % when $N=10, M \in \{10,20,30,40,50\}$ and $\tau^i(v) \in {2,5,10,20}$. Different color indicates the different task duration and the $x$ axis respresents the number of targets (M). 
    In Scene 2 in the maze map, the cost ratio quickly goes down to zero because the corresponding success rate is almost zero.
    \abbrCBXS finds cheaper solution than \abbrCBSS especially when $\tau$ is large.
    }
    \label{fig:CR}
    \vspace{-0.cm}
\end{figure}

\subsection{\abbrCBXS With Different Branching Rules}
\label{subsec:result2}
This section compares the number of conflicts resolved in \abbrCBXS by using different branching rules, namely the regular branching rule in CBSS~\cite{ren2023cbss} (denoted as old) and the new branching rule developed in this paper (denoted as new).
% According to the previous section, Scene 2 is more complex than Scene 1 and there is a bigger chance for the conflict in a target with task duration to occur (e.g. the cost ratio of Scene 2 in Fig.~\ref{fig:CR} is averagely larger than that of Scene 1). Thus, 
We run the tests in Scene 2 since it is more challenging according to the results in the previous section.
% to show the explicit advantage of our novel branching rule compared to the normal one. To collect more data, except for $M \in \{10,20,30,40,50\}\ and\ \tau^i(v) \in \{2,5,10,20\}$, 
In addition to varying $M$ and $\tau$, we also vary the number of agents $N \in \{5,10,20\}$.
For each $N$, we sum up the data for $M=\{10,20,30,40,50\}$ and $\tau = \{2,5,10,20\}$ and thus there are 500 instances for each $N$.
Let $N_c^{old}, N_c^{new}$ denotes the number of resolved conflicts using old and new branching rules respectively.

\begin{figure}[tb]
    \centering
    \includegraphics[width=0.98\linewidth]{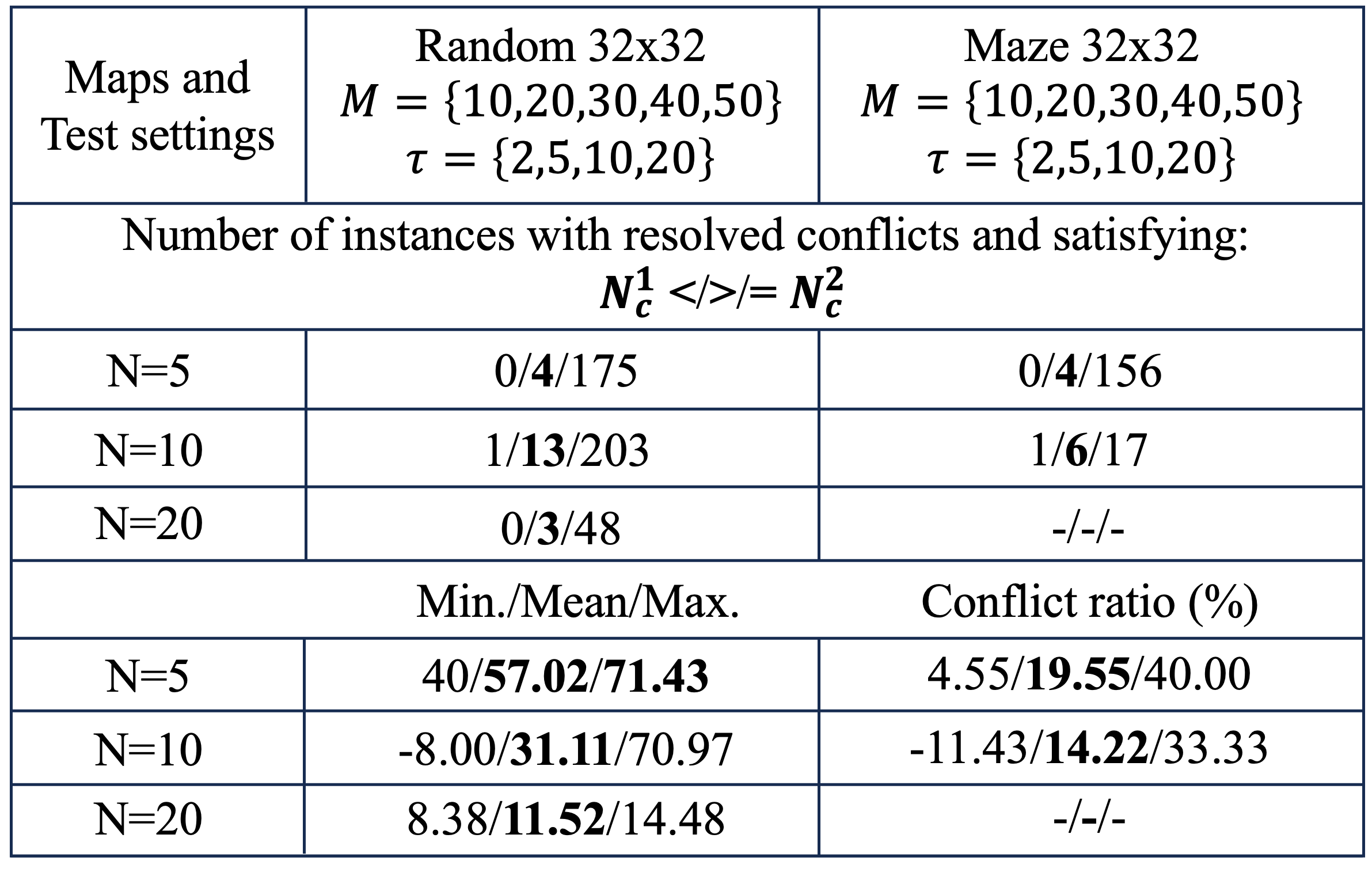}
    \vspace{-0.2cm}
    \caption{The number of resolved conflicts of \abbrCBXS with different branching rules. 
    % The maps and test settings are shown in row~1 and the two experiments begin at row~2 and row~6 respectively: 1) In the first experiment result from row~2 to row~5, it shows the comparison of the number of instances with resolved conflicts (i.e. the number of resolved conflicts in each instance is non-zero). For example, at row~3 and column~2, 0/\textbf{6}/216 represents that there are 222 instances with resolved conflicts in total, 0 instances where $N_c^1<N_c^2$, 6 instances where $N_c^1>N_c^2$, 216 instances where $N_c^1=N_c^2$. 2) In the second experiment result from row~6 to row~9, we revisit the instances where $N_c^1>N_c^2$ and report the conflict ratio (we will explain it later) in these instances for each $N$.
    The new branching rule is able to reduce the number of conflicts resolved during planning.
    }
    \label{fig:NC}
    \vspace{-0.1cm}
\end{figure}

% up to 80\%

As shown in Fig.~\ref{fig:NC}, $N_c^{new}$ is usually smaller than $N_c^{old}$, which means the new branching help reducing the number of conflicts resolved during the planning.
There are a few cases where $N_c^{old} < N_c^{new}$, which is due to tie breaking since we do not have a fixed rule to break ties.

We further introduce conflict ratio $({N_c^{old}-N_c^{new}})/{N_c^{old}}\times 100\%$ to compare the number of conflicts.
% \begin{equation}
%     conflict\ ratio = \frac{N_c^{old}-N_c^{new}}{N_c^{old}}\times 100\%
% \end{equation}
As shown in Fig.~\ref{fig:NC}, the new branching rule is able to reduce up to 70\% of the conflicts compared with the old branching rule.
The new rule is particularly beneficial with the long task duration.

\subsection{Gazebo Simulation}

\label{subsec:gazebo}
We run simulation in Gazebo for Scene 3 to execute the joint path planned by \abbrCBXS. Our simulation settings are $N=5, M=10, \tau \in [2, 10]$ as shown in Fig.~\ref{fig:simulation}.
% For trans
\begin{figure}[tb]
    \centering
    \includegraphics[width=0.98\linewidth]{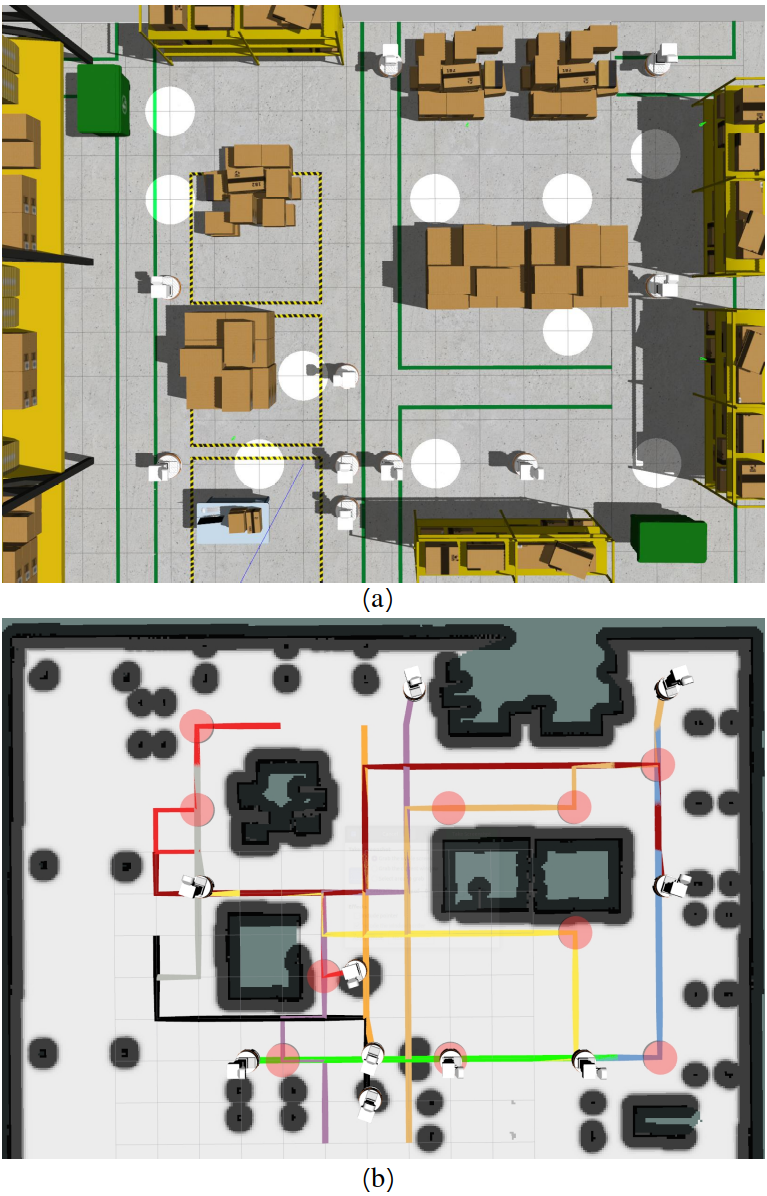}
    \caption{Simulation using Gazebo and path visualization with Rviz of Scene 3. (a) Gazebo simulation of a warehouse, multiple robots and target tasks (marked by white disks). (b) Rviz visualization of robots, joint path generated by \abbrCBXS and target tasks (marked by red disks).
    }
    \label{fig:simulation}
\end{figure}
% For conference
% \begin{figure}[tb]
%     \centering
%     \includegraphics[width=1\linewidth]{source/figures/gr.png}
%     \vspace{-0.8cm}
%     \caption{Simulation using Gazebo and path visualization with Rviz of Scene 3. (a) Gazebo simulation of a warehouse, multiple robots and target tasks (marked by white disks). (b) Rviz visualization of robots, joint path generated by \abbrCBXS and target tasks (marked by red disks).
%     }
%     \label{fig:simulation}
%     \vspace{-0.4cm}
% \end{figure}
During the execution, due to the disturbance and uncertainty in the robot motion, additional conflicts may occur when robots execute their paths. 
We thus implement a management system to monitor the position of all robots and coordinate their motion, by taking advantage of TPG in a similar way as described in \abbrTPGD.
Specifically, a TPG is first created based on the planned paths by \abbrCBXS.
When executing the path, nodes in TPG that has not precedence requirements (i.e., satisfying the conditions as described in \abbrTPGD) are deleted from the TPG and are sent to the robots for execution.
By doing so, the robots remain collision-free along their paths, even if a robot has unexpected delay or the task duration used by the planner mismatches the actual task execution time.
% when a robot reaches a vertex\footnote{We use a KD-Tree to encode each robot's path and the position along the path is represented by a list of indexes in KD-Tree. Each robot's position in its path is computed by finding the nearest index value in the KD-Tree.}, the coordinate system would check if it reaches with any conflict with other robots coming to the same vertex. For instance, if any conflict happens for the robot $i$ moving from its current vertex $v_a$ to the next vertex $v_b$, the coordinate system would stop the robot $i$ in its current vertex $v_a$ until other robots pass through $v_b$. 
We demonstrate the simulation experiment in the multimedia attachment.
	
	\section{Conclusion and Future Work}\label{matspf:sec:conclude}
	This article investigates a generalization of \abbrMCPF called \abbrMATSPF, where agents execute tasks at target locations with heterogeneous duration.
We developed two methods—\abbrCBSS and \abbrCBXS to handle \abbrMATSPF, analyzed their properties. verified them with simulation with up to 20 agents and 50 targets, and discussed their pros and cons.
% 1) For \abbrCBSS, it combined the existing CBSS and \abbrTPGD to find a conflict-free solution that satisfies all the assignment constraints and finishes all tasks with heterogeneous duration. But \abbrCBSS can not guarantee the optimality. 2) For \abbrCBXS, we modified the transformation method to generate the optimal target sequences for all agents and thus guarantees the optimality of the final solution. Besides, in \abbrCBXS, we developed a new branching rule to enhance the efficiency of conflict resolution. We also provided numerical result to show that \abbrCBXS always computes the solutions with less cost than \abbrCBSS and the new branching rule in \abbrCBXS reduces up to 80\% of resolved conflicts compared to old branching rule. 
% At last, we simulated the joint paths generated by both \abbrCBSS and \abbrCBXS in Gazebo to valid their feasibility, and reused \abbrTPGD 
% The simulation shows that our method are applicable to robots in the presence of motion disturbance.
For future work, we can consider time window constraints on tasks.
% assignment constraints. For instance, when an agent is executing a task in one vertex $v$, to avoid spatial disturbance motivated by~\cite{li2019multi}, any other agents may be prohibited to visit the adjacent locations of vertex $v$, which brings additional spatial constraints when executing the task. In this way, we should expand the branching rules to resolve multiple types of constraints as well as improve resolution efficiency.
	
    % \section*{Acknowledgments}
    % This material is based upon work supported by the National Science Foundation under Grant No. 2120219 and 2120529. Any opinions, findings, and conclusions or recommendations expressed in this material are those of the author(s) and do not necessarily reflect the views of the National Science Foundation.
    % \input{reference.bbl}
    % \bibliographystyle{IEEEtran}
    % \bibliography{reference}

    % \bibliographystyle{IEEEtran}
    % \bibliography{reference}
    % \bibliographystyle{plain}
    % \bibliography{source/references}
	
\end{document}